\documentclass[final,12pt]{jmlr}

\title{Infinitely Divisible Noise in the Low Privacy Regime}

\usepackage{times}
\usepackage{thm-restate}
\usepackage{graphicx}

\usepackage[symbol]{footmisc}
\renewcommand{\thefootnote}{\fnsymbol{footnote}}

\DeclareMathOperator{\E}{\operatorname{E}}
\DeclareMathOperator{\Var}{\operatorname{Var}}
\DeclareMathOperator{\DGam}{\zeta_\Gamma}
\DeclareMathOperator{\Du}{\zeta_u}

\newtheorem{claim}{Claim}
\newtheorem{note}{Note}

\graphicspath{ {./plots/} }

\author{%
 \Name{Rasmus Pagh} \Email{pagh@di.ku.dk}\\
 \addr BARC and University of Copenhagen
 \AND
 \Name{Nina Mesing Stausholm} \Email{nimn@itu.dk}\\
 \addr IT University of Copenhagen%
}

\jmlrvolume{}
\jmlryear{}
\jmlrproceedings{}{}

\begin{document}

\maketitle

\begin{abstract}
Federated learning, in which training data is distributed among users and never shared, has emerged as a popular approach to privacy-preserving machine learning.
Cryptographic techniques such as secure aggregation are used to aggregate contributions, like a model update, from all users.
A robust technique for making such aggregates differentially private is to exploit \emph{infinite divisibility} of the Laplace distribution, namely, that a Laplace distribution can be expressed as a sum of i.i.d.~noise shares from a Gamma distribution, one share added by each user.

However, Laplace noise is known to have suboptimal error in the low privacy regime for $\varepsilon$-differential privacy, where $\varepsilon > 1$ is a large constant. In this paper we present the first infinitely divisible noise distribution for real-valued data that achieves $\varepsilon$-differential privacy and has expected error that decreases exponentially with $\varepsilon$.
\end{abstract}


\section{Introduction}

\emph{Differential privacy}, a state-of-the-art privacy definition, is a formal constraint on randomized mechanisms for privately releasing results of computations. It gives a formal framework for quantifying how well the privacy of an individual, whose data is part of the input, is preserved. %
In recent years, differentially private algorithms for machine learning have been developed and made available, for example, in TensorFlow/privacy and the Opacus library for PyTorch. 
Using such algorithms ensures that the presence or absence of a single data record in a database does not significantly affect the distribution of the model produced.

Concurrently, the area of \emph{federated learning}~\citep{McMahanMRHA17} has explored how to carry out machine learning in settings where training data is distributed among $n$ users and never shared.
Cryptographic techniques such as secure aggregation are used to aggregate contributions, like a model update, from all users~\citep{BonawitzIKMMPRS17}.
This setup is used, for example, in the federated learning system run by Google on data from Android phones.\footnote{\url{https://ai.googleblog.com/2017/04/federated-learning-collaborative.html}}
A robust technique for making such aggregates differentially private is to exploit \emph{infinite divisibility} of the Laplace distribution, namely, that a Laplace distribution can be expressed as a sum of $n$ i.i.d.~noise shares from a Gamma distribution, one share added to the input of each user~\citep{GoryczkaX17}.
That is, if user $i$ holds $x_i\in [0,\Delta]$, the input provided to the secure aggregation is $x_i + \eta_i$, where $\eta_i$ is sampled from a suitable Gamma distribution.
Infinitely divisible noise is resistent to \emph{dropout}, where some users never contribute to the sum, if we set $n$ to be a lower bound on the number of fully participating users.

However, the Laplace noise needed for $\varepsilon$-differential privacy yields expected error $\Theta(\Delta/\varepsilon)$, which is not optimal in the ``low privacy regime'' when $\varepsilon \gg 1$.
\cite*{GengKOV15} and \cite*{GengV16} presented the \emph{Staircase} mechanism  (see Lemma \ref{lem:staircasemech}), which can be parameterized to obtain expected error $\Theta(\Delta e^{-\varepsilon/2})$ or variance $\Theta(\Delta^2 e^{-2\varepsilon/3})$, thus outperforming the Laplace mechanism for $\varepsilon$ larger than some constant.
However, the noise distribution used by this mechanism is not infinitely divisible, so it cannot replace Laplace noise in federated settings.

In this paper we present the first infinitely divisible noise distribution for real-valued data that achieves $\varepsilon$-differential privacy and has expected error that decreases exponentially with $\varepsilon$.
Our new noise distribution, the \emph{Arete}\footnote[1]{The name \emph{Arete} is inspired by the word arête (pronounced "ah-ray't"), which is a sharp-crested mountain ridge, while also a concept from Greek mythology, Arete (pronounced "ah-reh-'tay") referring to moral virtue and excellence: the notion of the fulfillment of purpose or function and the act of living up to one's full potential \citep{AreteDefWiki}.} distribution, has expected absolute value and variance exponentially decreasing in $\varepsilon$, and thus comparable to that of the Staircase distribution up to constant factors in~$\varepsilon$. 
Figure~\ref{fig:densities} illustrates the shape of each of the three distributions.
\begin{figure}[t]
\centering
\subfigure[Laplace distribution]{
\includegraphics[width=.31\textwidth]{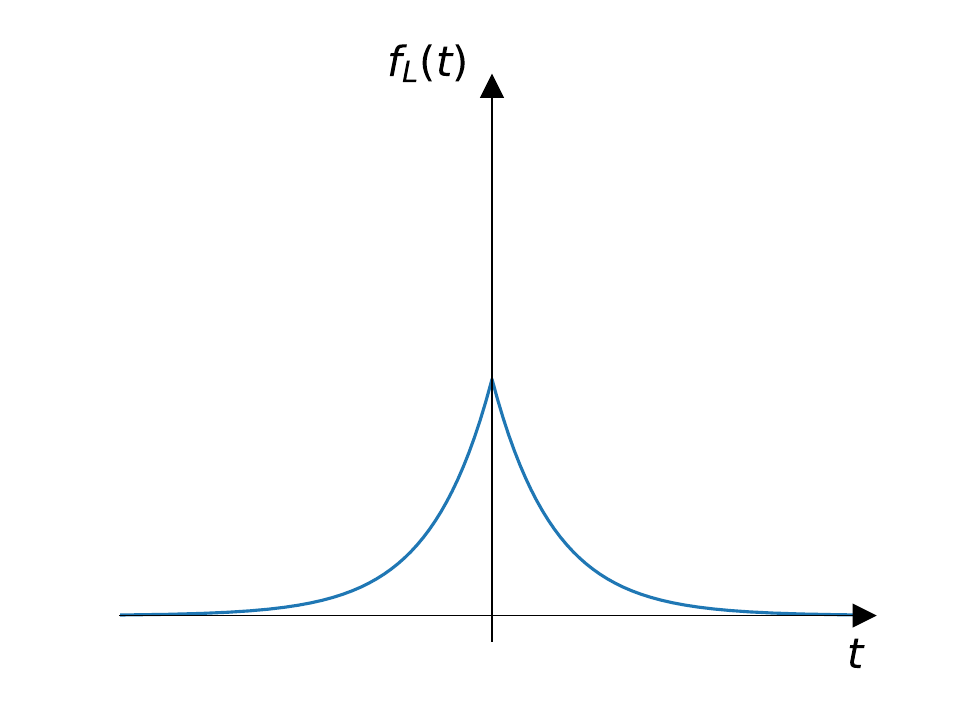}
}
\subfigure[Staircase distribution]{
\includegraphics[width=.31\textwidth]{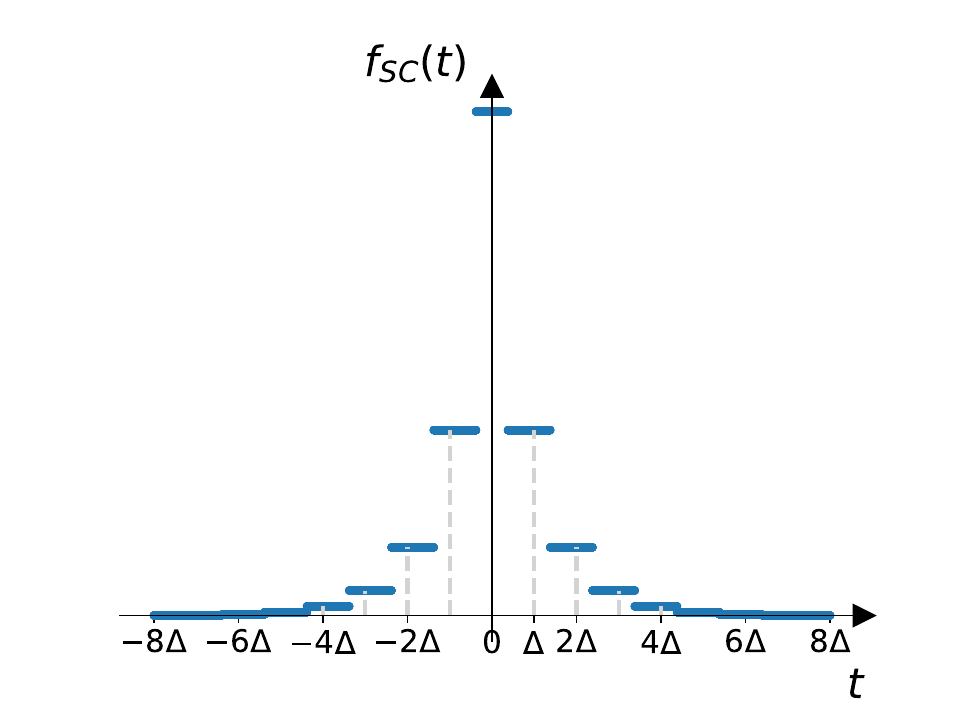}
}
\subfigure[Arete distribution]{
\includegraphics[width=.31\textwidth]{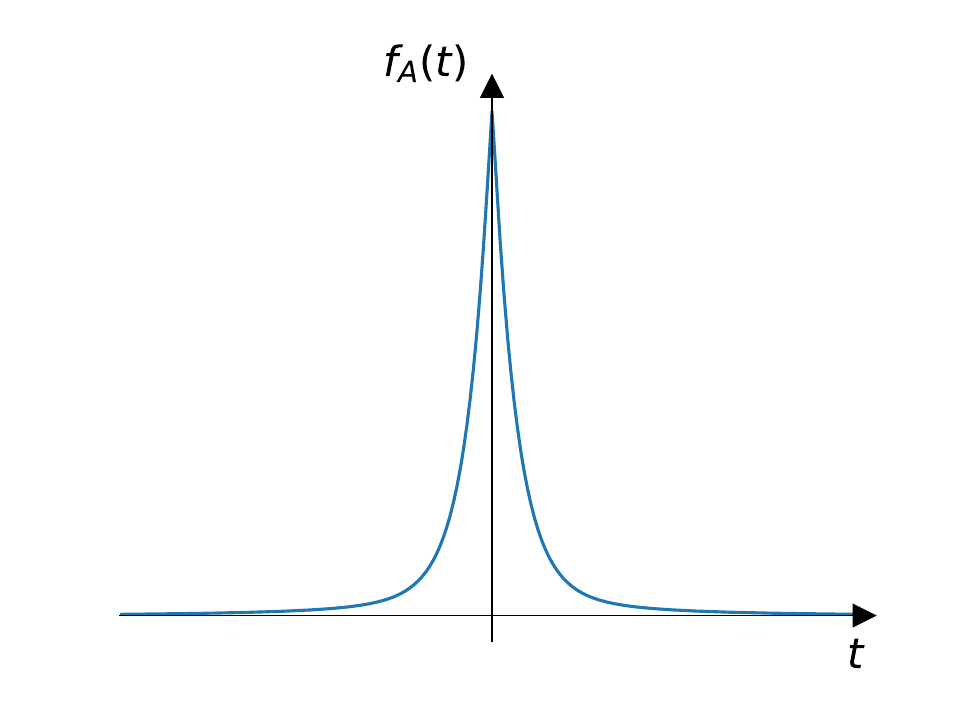}
}
\caption{\footnotesize{Illustration of the density functions for the Laplace, Staircase and Arete distributions.}}
\label{fig:densities}
\end{figure}
The Arete distribution has a continuous density function implying that the privacy level decreases more smoothly with sensitivity, in contrast to the Staircase distribution (see discussion in Section~\ref{sec:relatedwork}).

The main part of this paper is devoted to the accuracy and privacy analysis of the Arete mechanism. Appendix~\ref{sec:applications} discusses applications in distributed private data analysis.

\subsection*{The Arete Distribution}
For simplicity, we will limit ourselves to 1-dimensional setting. In order to deal with vectors (with $\ell_\infty$ sensitivity bounded by $\Delta$), we may simply add independent noise from the Arete distribution to each coordinate.
Our goal is to approximate the staircase distribution with an infinitely divisible distribution, so it is instructive to understand the essential properties of the staircase distribution: Only probability mass $\exp({-\Omega(\varepsilon)})$ is placed in the tails, which can be seen as a piece-wise uniform version of a scaled Laplace distribution. The majority of the probability mass is placed in a uniform distribution on an interval around zero of length $\exp({-\Omega(\varepsilon)})$. 

\begin{definition}[Arete distribution, informal]
\label{def:areteinformal}
Let independent random variables $X_1,X_2\sim\Gamma(\alpha,\theta)$ and $Y\sim Laplace(\lambda)$. Then $Z:=X_1-X_2+Y$ has Arete distribution with parameters $\alpha,\theta$ and $\lambda$, denoted $Arete(\alpha,\theta,\lambda)$. When the parameters $\alpha$, $\theta$ and $\lambda$ are understood from the context, we use $f_A(t)$, $t\in\mathbb{R}$, to denote the density function of $Z$.
\end{definition}

Since the $\Gamma$ and Laplace distributions are continuous and infinitely divisible, and the Laplace distribution is symmetric, it follows that the Arete distribution also has these properties. In Section~\ref{sec:mainAreteProperties} we show:
\begin{restatable}{lemma}{mainAreteProperties}
\label{lem:mainAreteProperties}
For any choice of parameters $\alpha,\theta,\lambda>0$, the $Arete(\alpha,\theta,\lambda)$ distribution is infinitely divisible and has density $f_A(t)$ that is continuous, symmetric around 0, and monotonely decreasing for $t>0$.
\end{restatable}

Next, we discuss the intuition behind the noise and privacy properties of the Arete distribution:
For privacy parameter $\varepsilon>0$ and sensitivity $\Delta> 0$ we concern ourselves with distributions $\mathcal{D}$ with support $S$ and density function $f_\mathcal{D}$ satisfying
\begin{align}
\label{eq:DPratio}
e^{-\varepsilon}\le \frac{f_\mathcal{D}(t)}{f_\mathcal{D}(t+a)}\le e^\varepsilon, \qquad \forall t, a\in \mathbb{R}, |a|\leq\Delta
\end{align}
as this property is sufficient to ensure differential privacy, which is our main goal. We will refer to the property (\ref{eq:DPratio}) as the \emph{differential privacy constraint}.
In order to minimize the magnitude of the noise, the goal is to find a distribution with minimal expected (absolute) value while satisfying (\ref{eq:DPratio}). 

The difference of two $\Gamma$ distributed random variables can be parameterized to have similar tails and to ``peak'' in an interval around zero of the same width as the staircase distribution. But this does not provide differential privacy since the density function has a singularity at zero. To achieve differential privacy we add a small amount of Laplace noise that ``smooths out'' the singularity.
In more detail, the $\Gamma(\alpha,\theta)$-distribution (see Definition \ref{def:gammadistr}) with shape $\alpha<1$, has most of its probability mass on an interval $(0,O(\alpha))$. The difference of two $\Gamma$ distributions does not satisfy (\ref{eq:DPratio}) for any choice of $\alpha<1$, as the density tends to infinity for values going to zero. To fix this we need to ``flatten the curve'' of the density function in the neighborhood of 0. Consider $Z':=X+Y$ for independent  $X\sim\Gamma(\alpha,\theta)$ and $Y\sim Exp(\lambda)$. The Exponential distribution, with a suitable choice of parameter $\lambda$, is used to flatten the density function of the $\Gamma$ distribution close to 0. In order to get a noise distribution that is symmetric around zero, we further consider $Z=X_1+Y_1-(X_2+Y_2)$ for $X_1,X_2\sim\Gamma(\alpha,\theta)$ and $Y_1,Y_2\sim Exp(\lambda)$. 
Our definition of the Arete distribution follows from the fact that if $Y_1,Y_2\sim Exp(\lambda)$, then $Y=Y_1-Y_2\sim Laplace(\lambda)$. We provide an explicit setting for the parameters $\alpha,\theta,\lambda$ in Lemma \ref{lem:main}.

We note that the Arete distribution generalizes the Laplace distribution in the sense that we obtain $Laplace(\lambda)$ as the limiting distribution for $\alpha \rightarrow 0$. In this sense, the Arete distribution is suitable also for $\varepsilon$ close to zero, where it may simply be used to implement Laplace noise.

\subsection*{Main Results}

\begin{figure}[t]
    \centering
    \includegraphics[width=0.45\linewidth]{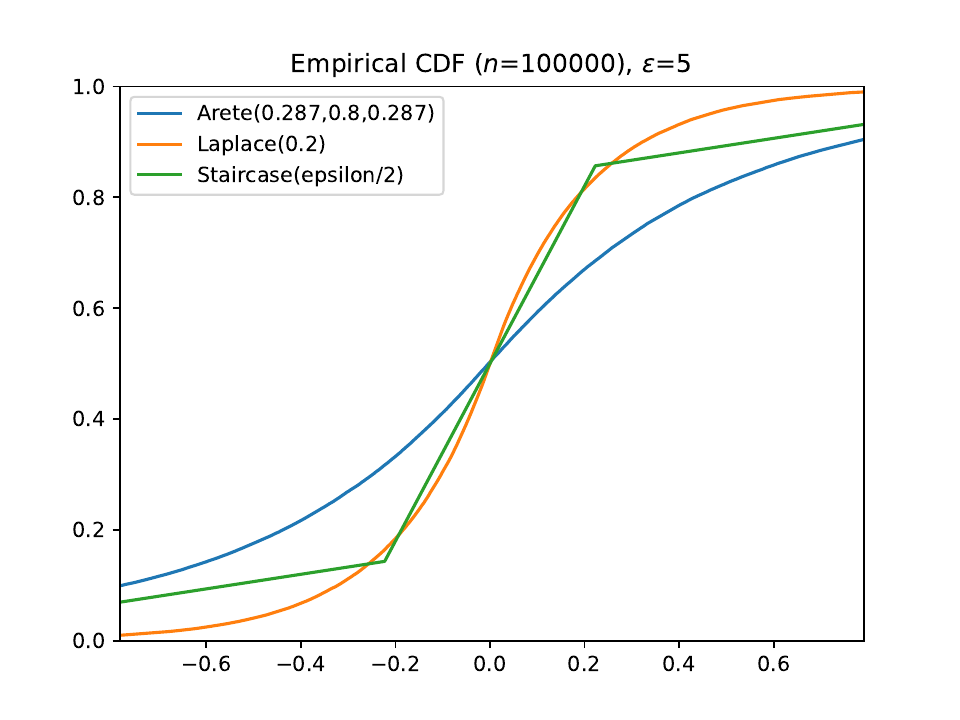}
    \includegraphics[width=0.45\linewidth]{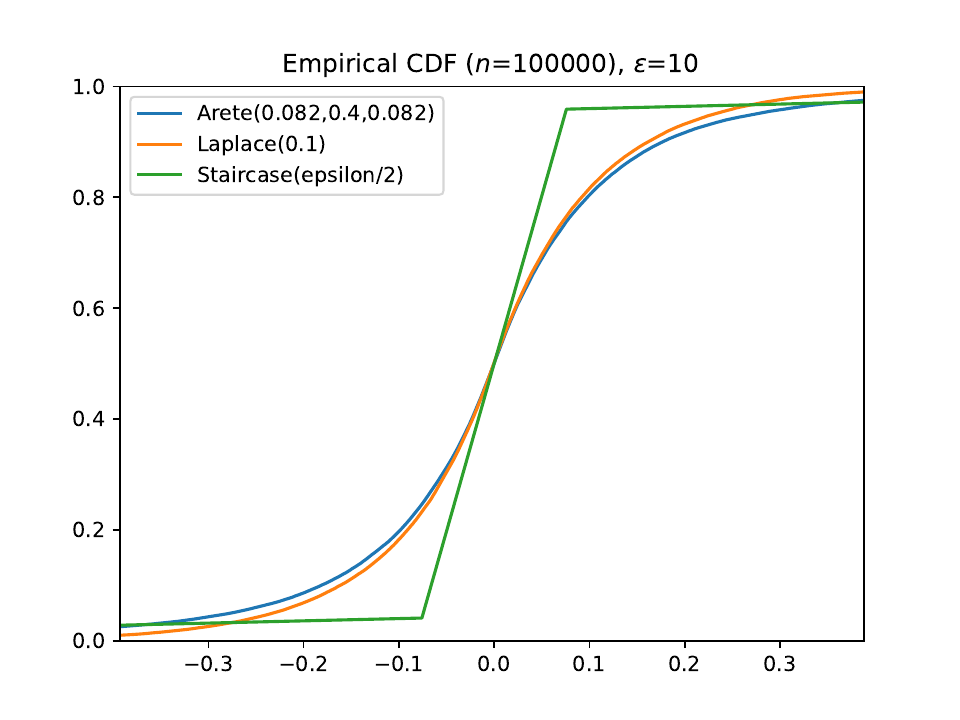}
    \includegraphics[width=0.45\linewidth]{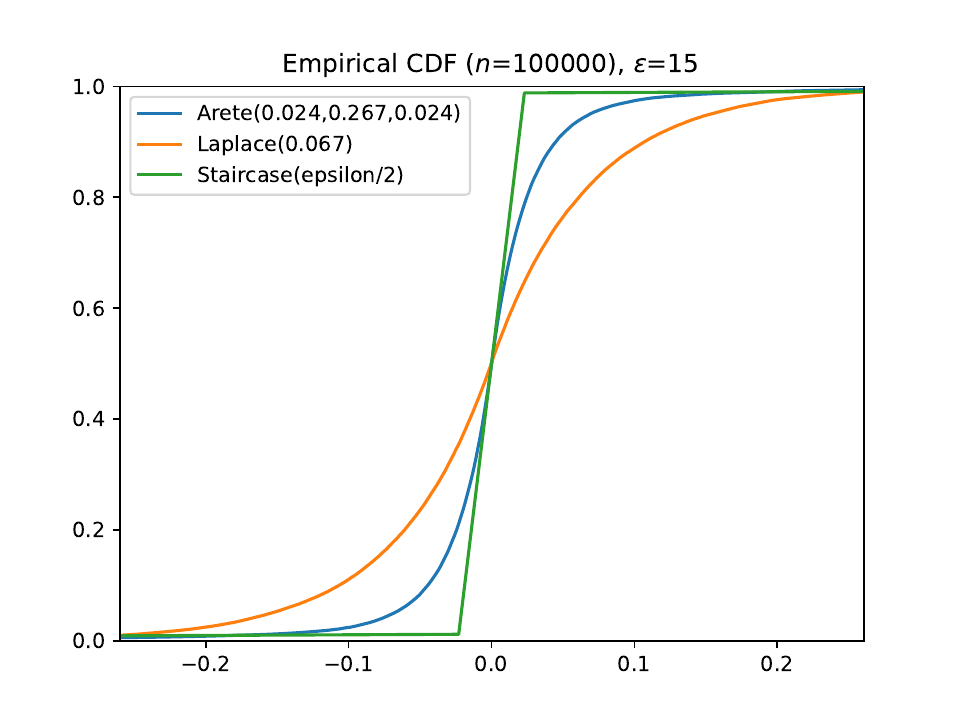}
    \includegraphics[width=0.45\linewidth]{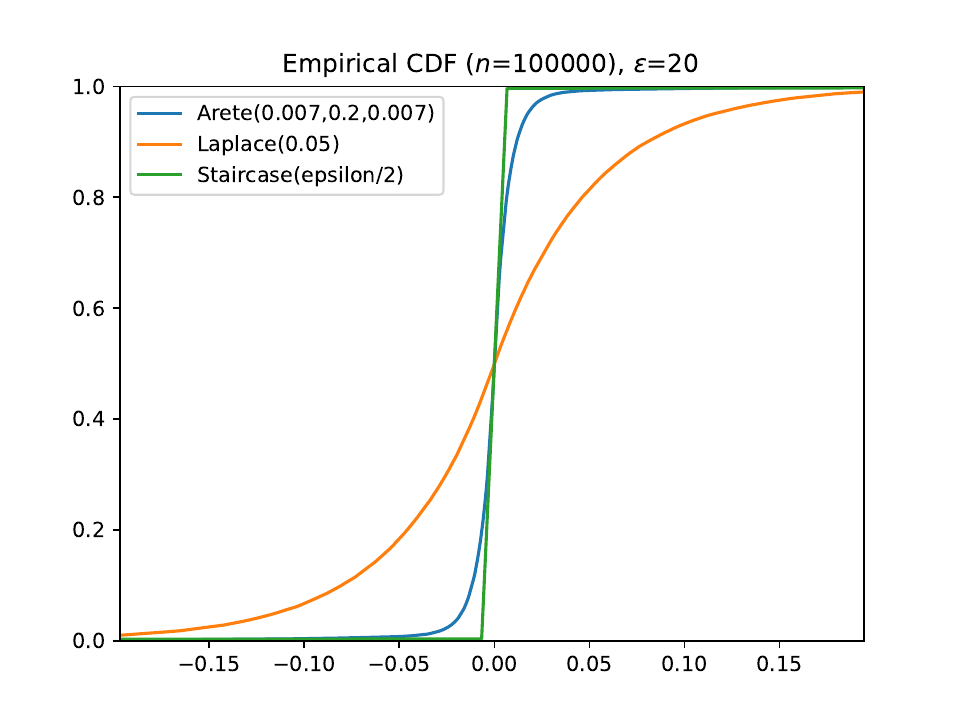}
    \caption{\footnotesize{Empirical cumulative distribution functions for Arete distributions parameterized according to Lemma~\ref{lem:main} for various values of $\varepsilon$, compared to two known $\varepsilon$-differentially noise distributions: Laplace($1/\varepsilon$), which is infinitely divisible, and a Staircase distribution (not infinitely divisible) parameterized to be $(\varepsilon + o(1))$-differentially private at neighbor distance $\Delta = 1+o(1)$. The values of $\varepsilon < 20$ are for illustration only, since they are too small for Lemma~\ref{lem:main} to apply.}}
    \label{fig:arete_cdf}
\end{figure}

Let the Arete distribution $Arete(\alpha,\theta,\lambda)$ be as in Definition \ref{def:areteinformal} (and formally, Definition \ref{def:aratedistr}) with density function $f_A$.
In Section~\ref{sec:puttingThingsTogether} we show the following result:
\begin{restatable}{lemma}{main}
\label{lem:main}
For every choice of $\Delta\ge 2/e$ and $\varepsilon\ge 20+4\ln(\Delta)$ there exist parameters $\alpha,\beta,\lambda>0$ such that:
\begin{itemize}
    \item For every choice of $t, a\in\mathbb{R}$ with $|a|\leq\Delta$, $e^{-\varepsilon}\le \frac{f_A(t)}{f_A(t+a)}\le e^\varepsilon$:  %
    \item For $Z\sim Arete(\alpha,\theta,\lambda)$, $\E[\vert Z\vert]=O(\Delta e^{-\varepsilon/4})$ and $\Var[Z]=O(\Delta^2e^{-\varepsilon/4})$.
\end{itemize}
Parameters $\alpha=e^{-\varepsilon/4}, \theta=\frac{4\Delta}{\varepsilon}$ and  $\lambda=e^{-\varepsilon/4}$ suffice.
\end{restatable}%

The condition $\Delta\ge 2/e$ is not essential in the sense that we may always scale the noise by a factor $\Delta$, which means that it is enough to consider sensitivity~1. This also means that $\varepsilon\ge 20$ suffices if we use a scaled Arete distribution.

Figure~\ref{fig:arete_cdf} shows empirical cumulative distribution functions for Arete distributions derived from Lemma~\ref{lem:main}. The code for generating these plots, as well as all other plots in this paper, can be found on GitHub.\footnote{\url{https://github.com/rasmus-pagh/alt22-code}}

\medskip

The following corollary, shown in Section~\ref{sec:mainAreteProperties}, says that adding noise from the Arete distribution gives an $\varepsilon$-differentially private mechanism. We refer to Section~\ref{sec:DP} for details about differential privacy and the definition of the sensitivity of a query.
\begin{definition}[The Arete mechanism]
\label{def:aretemech}
Let $x\in\mathcal{X}^d$ be an input and $q:\mathcal{X}^d\rightarrow \mathbb{R}$ a query with sensitivity bounded by $\Delta\ge 2/e$. Given parameters $\alpha, \theta,\lambda$, the \emph{Arete mechanism} $\mathcal{M}_{Arete}(x)$ samples $Z\sim Arete(\alpha,\theta,\lambda)$ and returns $q(x)+Z$. %
\end{definition}
\begin{restatable}{corollary}{mainDP}
\label{cor:mainDP}
The Arete mechanism $\mathcal{M}_{Arete}$ with parameters as specified in Lemma \ref{lem:main} has expected error $O(\Delta e^{-\varepsilon/4})$ and is $\varepsilon$-differentially private. %
\end{restatable}

\medskip

{\bf Discussion of Large Values of $\varepsilon$.}
\label{sec:largeeps}
Values of $\varepsilon$ larger than one often appear in practice. Examples of deployments using large values of~$\varepsilon$ include Google's RAPPOR with $\varepsilon$ up to 9, Apple's MacOS with $\varepsilon=6$, iOS 10 with $\varepsilon=14$ \citep{Greenberg}, and US Census Bureau with $\varepsilon$ up to $19.6$.\footnote{{\scriptsize\url{https://www.census.gov/newsroom/press-releases/2021/2020-census-key-parameters.html}}} %

These examples are not directly comparable to our setting since they deal with either local differential privacy (data of a single user), or with the release of many statistics (rather than a single aggregate). 
However, we note that mechanisms in the low-privacy regime can often be ``boosted'', e.g.~using sampling~\citep{BalleBG18} or shuffling~\citep{ErlingssonFMRTT19}, to obtain a mechanism with a better privacy parameter.
If sampling is used to improve privacy we will of course get noise due to sampling error, and this will dominate the total error in most cases.
One way to use our result (with secure aggregation) is to essentially match the error of estimating a sum from a sample, getting privacy almost for free when the sample is much smaller than the whole data set.

The lower bound on $\varepsilon$ in Lemma \ref{lem:main} is high, but we note that \emph{empirically} we achieve differential privacy for significantly lower values of $\varepsilon$ (see Section~\ref{sec:conclusion}).

\section{Related Works}
\label{sec:relatedwork}
A fundamental question is what can be said about the tradeoff between error and privacy. \cite*{HardtT10} study this tradeoff for linear queries, showing a lower bound of $\Omega(1/\varepsilon)$ for worst case expected $\ell_2$-norm of noise (std. deviation) %
under the constraint of $\varepsilon$-differential privacy for small $\varepsilon$. 
\cite{NikolovTZ13} extend the work of \citep{HardtT10} to the tradeoff between error and $(\varepsilon,\delta)$-differential privacy.
For error that can be a general function of the added noise, \cite*{GupteS10} and \cite*{GhoshRS12} introduced the Geometric Mechanism for \emph{counting queries} (integer valued) with sensitivity~1, showing that the optimal noise has a (symmetric) Geometric distribution with error (standard deviation) $\Theta(e^{-\varepsilon/2})$. %
\cite{BrennerN10} extend these results, showing that for general queries there is no optimal mechanism for $\varepsilon$-differential privacy.
In the high privacy regime, \cite{GengV16a} present a (near) optimal mechanism for integer-valued vector queries for $(\varepsilon,\delta)$-differential privacy, achieving error (for single-dimensional queries) $\Theta(\min\{1/\varepsilon,1/\delta\})$ for small $\varepsilon$ and $\delta$. %
Though the geometric mechanism yields optimal error in the \emph{discrete} setting, and is infinitely divisible~\citep{GoryczkaX17}, it does not seem to generalize to a differentially private, infinitely divisible noise distribution in the real-valued setting.
Recently, \cite{KairouzL021} studied a distributed, discrete version of the Gaussian mechanism, which has good composition properties.
This mechanism is not aimed at the low privacy regime, and does not have error that decreases exponentially as $\varepsilon$ grows.

Generalizing to \emph{real-valued} 1-dimensional queries with arbitrary sensitivity, \cite{GengV16} introduced the $\varepsilon$-differentially private Staircase mechanism (see Lemma \ref{lem:staircasemech}), which adds noise from the Staircase distribution -- a geometric mixture of uniform distributions. 
The density function of the Staircase distribution, $f_{SC}$, is a piece-wise continuous step (or ``staircase-shaped'') function, symmetric around zero, with geometrically decaying density as a function of the distance from zero.
The staircase mechanism circumvents a lower bound of \cite{KoufogiannisHP15} on ``Lipschitz'' differential privacy, which requires the privacy loss to be bounded by $\varepsilon |q(x)-q(y)|/\Delta$, by only bounding the worst case privacy loss for $|q(x)-q(y)|/\Delta \leq 1$.

\cite{GengV16} prove that the optimal $\varepsilon$-differentially private mechanism for single real-valued queries, measuring error as expected magnitude or variance of noise, is not Laplace but rather Staircase distributed: while the Laplace mechanism is asymptotically optimal as $\varepsilon\rightarrow 0$, the Staircase mechanism performs better in the low privacy regime (i.e., for large $\varepsilon$), as the expected magnitude of the noise is exponentially decreasing in $\varepsilon$. Specifically, for sensitivity $\Delta$ and for the parameter setting of $\gamma$ optimizing for expected noise magnitude, the Staircase mechanism achieves error $\Theta(\Delta e^{-\varepsilon/2})$. For the choice of $\gamma$ optimizing for variance, the Staircase mechanism ensures variance of the noise $\Theta(\Delta^2 e^{-2\varepsilon/3})$. We note that the $\gamma$ optimizing for noise magnitude is not generally the same a for optimizing for variance. The Laplace distribution has expected noise magnitude $\Theta(\Delta/\varepsilon)$ and variance $\Theta(\Delta^2/\varepsilon^2)$. In comparison, the expected noise magnitude and variance of the Arete distribution is also exponentially decreasing in $\varepsilon$, specifically $O(\Delta e^{-\varepsilon/4})$ and $O(\Delta^2e^{-\varepsilon/4})$, respectively, for our choice of parameters. The expected error and variance mentioned here are for a single parameter setting for both Laplace and Arete mechanisms.

As we want a noise distribution that is implementable in a distributed setting, we limit our interest to noise distributions that are oblivious of the input data and the query output. %
A nice property of the Arete distribution is that the density function is continuous and so we get a more graceful decrease in privacy than the Staircase mechanism for inputs that are not quite neighboring. 
We can measure this using the \emph{worst case privacy loss}, which is the logarithm of the largest ratio between the mechanism's density functions on inputs $x$ and $y$.
If $|q(x)-q(y)| \leq \Delta$, inequality (\ref{eq:DPratio}) implies that the worst case privacy loss is at most $\varepsilon$.
For inputs with $|q(x)-q(y)| > \Delta$ this is no longer guaranteed, but it is still interesting to study how the level of privacy decreases as a function of $|q(x)-q(y)|$.
The Staircase mechanism is \emph{exactly} fitted to the sensitivity of the query such that differential privacy is guaranteed for neighboring inputs, but as soon as $|q(x)-q(y)| > \Delta$ the worst case privacy loss is immediately doubled. The privacy loss increases in a smoother fashion when applying the Arete distribution due to the continuity of the density function (See Figure~\ref{fig:privacyloss}).

\begin{figure}[t]
\centering
\subfigure[Staircase mechanism]{
\includegraphics[width=.45\textwidth]{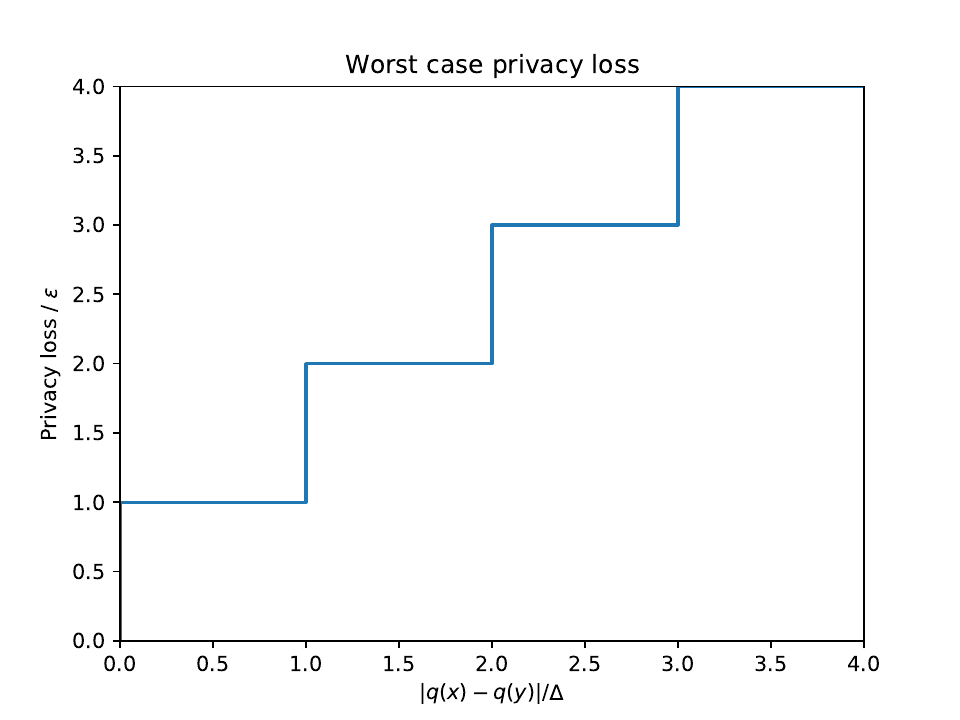}
}
\subfigure[Arete mechanism]{
\includegraphics[width=.45\textwidth]{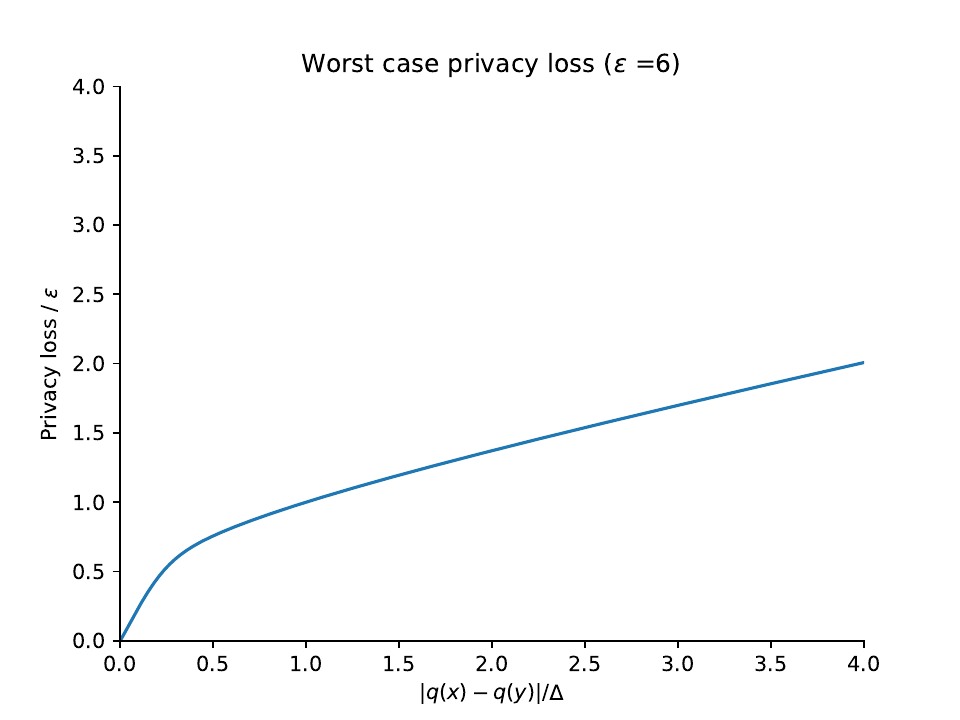}
}
\caption{\footnotesize{Worst case privacy loss of the Staircase and Arete mechanisms (the latter with $\varepsilon = 6$) as a function of the difference between query outputs. As we have no closed form for the density of the Arete distribution, the plot on the right hand side is a numerical approximation computed by discretizing the involved $\Gamma$ and Laplace distributions and computing the convolution of the discretizations. The privacy loss of the Arete distribution increases rapidly from distance $0$ to about $\Delta / 2$ and then increases more slowly from $\Delta / 2$ to $\Delta$ where it reaches $\varepsilon$.
As can be seen the Arete distribution has a smaller privacy loss than the staircase distribution for larger distances $|q(x)-q(y)|$, suggesting that it in fact adds more noise than necessary.}}
\label{fig:privacyloss}
\end{figure}
\cite{GengKOV15} extend the Staircase mechanism from  \citep{GengV16} to queries in multiple dimensions.

\section{The Arete Distribution}
This section introduces the Arete distribution and some of its properties.
We refer to Appendix~\ref{sec:basics} for definitions of probability distributions and differential privacy.
The following lemma is well-known from the probability theory literature:
\begin{lemma}
\label{lem:convolution}
If $X$ and $Y$ are independent, continuous random variables with density functions $f_X$ and $f_Y$, %
then $Z=X+Y$ is a continuous random variable where 
the density is the convolution
\[
f_Z(z)=\int_{-\infty}^\infty f_X(x)f_Y(z-x)dx\textbf=\int_{-\infty}^\infty f_X(z-x)f_Y(x)dx.
\]
\end{lemma}

The following distribution will be useful in defining the Arete distribution:
\begin{definition}[The $\Gamma-\Gamma$ Distribution]
\label{def:gammagammadistr}
Let $X_1,X_2\sim \Gamma(\alpha,\theta)$ be independent and consider their difference $X~:=~X_1~-~X_2$. We say that $X$ has the $\Gamma-\Gamma(\alpha,\theta)$ distribution and the density of $X$ is
\[
f_{\Gamma-\Gamma(\alpha,\theta)}(t)=\int_{0}^\infty f_{\Gamma(\alpha,\theta)}(t+x)f_{\Gamma(\alpha,\theta)}(x)dx=\begin{cases}
\int_{0}^\infty f_{\Gamma}(t+x)f_{\Gamma}(x)dx,\qquad t\ge 0\\
\int_{\vert t\vert}^\infty f_{\Gamma}(t+x)f_{\Gamma}(x)dx,\qquad t<0.
\end{cases}
\]
where the integrals are reduced to the intervals where $f_{\Gamma}(t+x)f_{\Gamma}(x)$ is non-zero.
\end{definition}
\begin{lemma}
\label{lem:infdivGammagamma}
The $\Gamma-\Gamma$ distribution is infinitely divisible: For $2n$ independent random variables $X_i, Y_i\sim \Gamma(\alpha/n,\theta)$, we have $X=\sum_{i=1}^n(X_i-Y_i)\sim \Gamma-\Gamma\left(\alpha, \theta\right)$.
\end{lemma}
\begin{proof}
The result follows immediately from infinite divisibility of the $\Gamma$-distribution.
\end{proof}

\begin{definition}[The Arete distribution]
\label{def:aratedistr}
Let $X\sim \Gamma-\Gamma(\alpha,\theta)$ and $Y\sim Laplace(\lambda)$ be independent. Define $Z:=X+Y$, then $Z\sim Arete(\alpha,\theta,\lambda)$ for $\alpha,\theta,\lambda>0$. The density of $Z$ is 
\[
f_{A(\alpha,\theta,\lambda)}(t)=\int_{-\infty}^\infty f_{\Gamma-\Gamma(\alpha,\theta)}(t-x)f_{L(\lambda)}(x)dx=\int_{-\infty}^\infty f_{L(\lambda)}(t-x)f_{\Gamma-\Gamma(\alpha,\theta)}(x)dx,\quad t\in\mathbb{R}.
\]
\end{definition}
\begin{lemma}
\label{lem:infdivArete}
The Arete distribution is infinitely divisible: For $4n$ independent random variables $X_{1i}, X_{2i}\sim \Gamma(\alpha/n,\theta)$ and $Y_{1i}, Y_{2i}\sim \Gamma(1/n,\lambda)$, we have $X=\sum_{i=1}^n(X_{1i}-X_{2i}+(Y_{1i}-Y_{2i}))\sim Arete\left(\alpha, \theta,\lambda\right)$.
\end{lemma}
\begin{proof}
The result follows immediately from infinite divisibility of the Laplace distribution and Lemma \ref{lem:infdivGammagamma}.
\end{proof}

\begin{note}
\label{note:densityofGammaGamma}
We remark that we are only interested in $0<\alpha<1$. Furthermore, we do not explicitly state the density of the Arete distribution, as there is no simple closed form for the density of the $\Gamma-\Gamma$ distribution. (It can, however, be expressed in terms of Bessel functions -- see \citep{mathai1993noncentral}.) A similar intuitive way of defining our distribution would be to use a symmetric version of the $\Gamma$-distribution (two halved $\Gamma$-distributions put back-to-back at zero), instead of the $\Gamma-\Gamma$-distribution. An important property of our distribution is infinite divisibility such that we can draw independent noise shares that sum to a random variable following the Arete distribution. As opposed to our $\Gamma-\Gamma$ distribution, it is not clear whether a symmetric $\Gamma$-distribution is infinitely divisible.
\end{note}

\subsection{Symmetric Density Functions}
We observe some simple properties of the Arete distribution.

\begin{restatable}{lemma}{integralsymmetricfunctions}
\label{lem:integralsymmetricfunctions}
For $f,g:\mathbb{R}\rightarrow\mathbb{R}$, that are symmetric around 0, i.e., $f(x)=f(-x)$ and $g(x)~=~g(-x)$, we have for any $t\in\mathbb{R}$
\[
\int_{-\infty}^{\infty}f(x)g(t-x)dx=\int_{-\infty}^{\infty}f(x)g(\vert t\vert-x)dx.
\]
In particular, the convolution $f\ast g$ is symmetric around 0.
\end{restatable}
\begin{proof}
The statement is immediate for $t\ge 0$, so suppose $t<0$. Then for any $a,b\in\mathbb{R}$
\[
\int_{-a}^{-b}f(x)g(t-x)dx=\int_{-a}^{-b}f(x)g(\vert t\vert+x)dx=\int_{b}^{a}f(-x)g(\vert t\vert-x)dx=\int_{b}^{a}f(x)g(\vert t\vert-x)dx
\]
where first step is by symmetry of $g$, the second step follows from integration by substitution and the last step is by symmetry of $f$.
In particular, we may let $a$ and $b$ be $\pm\infty$.
\end{proof}

\begin{restatable}{lemma}{GammaGammaSymmetric}
\label{lem:GammaGammaSymmetric}
$f_{\Gamma-\Gamma}$ is symmetric around 0.
\end{restatable}
\begin{proof}
We prove that $f_{\Gamma-\Gamma}(t)=f_{\Gamma-\Gamma}(\vert t\vert)$ for all $t\in\mathbb{R}$. Clearly, this is the case if $t\ge 0$, so suppose $t<0$. 
By Definition \ref{def:gammagammadistr}
\[
f_{\Gamma-\Gamma}(t)=\int_{\vert t\vert}^\infty f_{\Gamma}(t+x)f_{\Gamma}(x)dx=\int_{\vert t\vert}^\infty f_{\Gamma}(x-\vert t\vert)f_{\Gamma}(x)dx=\int_{0}^\infty f_{\Gamma}(x)f_{\Gamma}(\vert t\vert+x)dx=f_{\Gamma-\Gamma}(\vert t\vert).
\]
where the penultimate step follows from integration by substitution with $x-\vert t\vert$.
\end{proof}

\begin{corollary}
\label{cor:symmetryArete}
$f_A$ is symmetric around 0. %
\end{corollary}
\begin{proof}
The result follows directly from symmetry of the density of the Laplace distribution, $f_L$, and Lemmas \ref{lem:integralsymmetricfunctions} and \ref{lem:GammaGammaSymmetric}.
\end{proof}

\subsection{Properties of the Arete Distribution}
\label{sec:mainAreteProperties}
We restate the lemma here for convenience:
\mainAreteProperties*
\begin{proof}
Symmetry of the density function $f_A$ is proven in Corollary \ref{cor:symmetryArete} and infinite divisibility in Lemma \ref{lem:infdivArete}. Since $f_\Gamma$ and $f_L$ are continuous, $f_{\Gamma-\Gamma}$ and $f_A$ are also continuous by Lemma \ref{lem:convolution}. %
We prove that $f_A$ is monotonely decreasing, i.e., for $\vert t\vert\le \vert t'\vert$ we have $f_A(t)\ge f_A(t')$.
First, we argue that $f_{\Gamma-\Gamma}$ is monotonely decreasing. Recall Definition \ref{def:gammagammadistr} and observe
\[
f_{\Gamma-\Gamma}(t)=\int_{0}^\infty f_{\Gamma}(\vert t\vert+x)f_\Gamma(x)dx,\qquad \forall t\in\mathbb{R}
\]
which is immediate for $t\ge 0$ while for $t<0$
\[
f_{\Gamma-\Gamma}(t)=\int_{\vert t\vert}^\infty f_{\Gamma}(-\vert t\vert+x)f_\Gamma(x)dx=\int_{0}^\infty f_{\Gamma}(x')f_\Gamma(x'+\vert t\vert)dx'
\]
where we substituted $x':=x-\vert t\vert$.
So assume $\vert t\vert\le \vert t'\vert$. Then, since $f_\Gamma$ is monotonely decreasing
\begin{align*}
    f_{\Gamma-\Gamma}(t)=\int_{0}^\infty f_{\Gamma}(\vert t\vert+x)f_\Gamma(x)dx\ge\int_{0}^\infty f_{\Gamma}(\vert t'\vert+x)f_\Gamma(x)dx=f_{\Gamma-\Gamma}(t').
\end{align*}
We prove that $f_A$ is also monotonely decreasing: 
Assuming that $\vert t\vert\le \vert t'\vert$ we prove that $f_A(t)~\ge~f_A(t')$. Recall Definition \ref{def:aratedistr} and observe
\[
f_A(t)=\int_{-\infty}^\infty f_{\Gamma-\Gamma}(\vert t\vert-x)f_L(x)dx=\int_{-\infty}^\infty f_{\Gamma-\Gamma}(x)f_L(\vert t\vert-x)dx
\]
which is obvious for $t\ge0$ and since for $t<0$:
\begin{align*}
    f_A(t)=\int_{-\infty}^\infty f_{\Gamma-\Gamma}(t-x)f_L(x)dx=\int_{-\infty}^\infty f_{\Gamma-\Gamma}(\vert t\vert+x)f_L(x)dx=\int_{-\infty}^\infty f_{\Gamma-\Gamma}(x')f_L(\vert t\vert -x')dx'
\end{align*}
using that $f_{\Gamma-\Gamma}$ and $f_L$ are symmetric and a substitution with $x':=\vert t\vert+x$. A similar argument can be made if the convolution is flipped. We conclude that
\begin{align*}
    f_A(t)=\int_{-\infty}^\infty f_{\Gamma-\Gamma}(\vert t\vert-x)f_L(x)dx\ge \int_{-\infty}^\infty f_{\Gamma-\Gamma}(\vert t'\vert-x)f_L(x)dx=f_A(t')
\end{align*}
using that $f_{\Gamma-\Gamma}$ is monotonely decreasing.
\end{proof}

We finally assume Lemma \ref{lem:main} and prove Corollary \ref{cor:mainDP}, restated here for convenience. The proof of Lemma \ref{lem:main} is given in Section~\ref{sec:proofofmainlemma}.
\mainDP*
\begin{proof} 
	The expected error bound follows directly from the bound on $\E[|Z|]$ in Lemma~\ref{lem:main}.
	For the claim of differential privacy, let $x,x'\in\mathbb{R}$ with $\vert x-x'\vert\le 1$. We show that for any subset $S\subset \mathbb{R}$ 
\begin{align}
\label{eq:proveDP}
\Pr[\mathcal{M}_{Arete}(x)\in S]\le e^\varepsilon \Pr[\mathcal{M}_{Arete}(x')\in S].
\end{align}
Let noise $Z\sim Arete(\alpha,\theta,\lambda)$ for parameters $\alpha,\theta,\lambda$ as in Lemma \ref{lem:main}. \\Define $S'~:=~S-q(x)~=~\{s~-~q(x):~s\in S\}$, then:
\begin{align*}
    \frac{\Pr[\mathcal{M}_{Arete}(x)\in S]}{\Pr[\mathcal{M}_{Arete}(x')\in S]}=\frac{\int_{S'}f_A(z) dz}{\int_{S'}f_A(z+q(x')-q(x)) dz}\le \frac{\int_{S'}f_A(\vert z\vert) dz}{\int_{S'}f_A(\vert z\vert+\vert q(x')-q(x)\vert) dz}
\end{align*}
where we used symmetry of $f_A$, the triangle inequality, and the fact that $f_A(t)$ is decreasing for $t > 0$.
By assumption $\vert q(x)-q(x')\vert\le \Delta$. Lemma \ref{lem:main} says that $f_A(t)/f_A(t+a)\le e^\varepsilon$ for all $t\in \mathbb{R}$ and $a \leq \Delta$, and so we get
\begin{align*}
   \frac{\int_{S'}f_A(\vert z\vert) dz}{\int_{S'}f_A(\vert z\vert+\vert q(x')-q(x)\vert) dz}\le \frac{\int_{S'}e^\varepsilon f_A(\vert z\vert+\Delta) dz}{\int_{S'}f_A(\vert z\vert+\vert q(x')-q(x)\vert) dz}\le e^\varepsilon.
\end{align*}
The first inequality in (\ref{eq:proveDP}) follows by symmetry. %
\end{proof}

\section{Proof of Main Lemma}
\label{sec:proofofmainlemma}
In the remaining part of this section we prove a number of theoretical lemmas, that will help prove our main result, Lemma \ref{lem:main}.
The bulk of the analysis is the proof of the first bullet point of Lemma~\ref{lem:main}, showing that the given parameters $\alpha, \theta,\lambda>0$ suffice to bound $f_A(t)/f_A(t+a)$ for all $t, a\in\mathbb{R}$, $|a|\leq \Delta$. We break this part of the analysis down in this section. The intuition behind the structure is as follows: We first remark that (to the best of our knowledge) there is no simple expression for the density of the $\Gamma-\Gamma$ distribution (see Note \ref{note:densityofGammaGamma}). Hence, we will show upper and lower bounds for $f_{\Gamma-\Gamma}$ and use these to bound the ratio $f_A(t)/f_A(t+a)$. 
As discussed earlier, we have not optimized for constants, and as our proof includes \emph{several} steps of bounding, our analysis may not be tight, thus leading to the high value of $\varepsilon$ required in Lemma \ref{lem:main}. A tighter analysis is likely to allow for a better setting of parameters $\alpha,\theta,\lambda$ and a smaller $\varepsilon$.
We give the proof of Lemma \ref{lem:main} in Section~\ref{sec:puttingThingsTogether}. %

\subsection{Bounds on Density of \texorpdfstring{$\Gamma-\Gamma$}{Gamma Minus Gamma} Distribution}
We first derive upper and lower bounds on the density function of the $\Gamma-\Gamma$ distribution (see Section~\ref{sec:basics} for definitions).
\begin{lemma}
\label{lem:boundsonGammaGamma}
For any $t\in\mathbb{R}$ and any $\DGam>0$
\[
f_{\Gamma}(\vert t\vert+\DGam)c_{\DGam}\le f_{\Gamma-\Gamma}(t)\le f_{\Gamma}(\vert t\vert)\qquad \text{where}\qquad c_{\DGam}:=\int_{0}^{\DGam}f_{\Gamma}(x)dx  \enspace .
\]
\end{lemma}
\begin{proof}
Recall Definition \ref{def:gammagammadistr} and Lemma \ref{lem:GammaGammaSymmetric} and let $t\in\mathbb{R}$. For the upper bound, we have
\begin{align*}
    f_{\Gamma-\Gamma}(t)&=\int_{0}^\infty f_{\Gamma}(\vert t\vert+x)f_\Gamma(x)dx< f_{\Gamma}(\vert t\vert)\int_{0}^\infty f_\Gamma(x)dx=f_{\Gamma}(\vert t\vert).
\end{align*}

\noindent
For the lower bound we have for any $\DGam>0$
\begin{align*}
f_{\Gamma-\Gamma}(t)=
\int_{0}^{\infty}f_{\Gamma}(\vert t\vert+x)f_{\Gamma}(x)dx\ge \int_{0}^{\DGam}f_{\Gamma}(\vert t\vert+x)f_{\Gamma}(x)dx\ge f_{\Gamma}(\vert t\vert+\DGam)\int_{0}^{\DGam}f_{\Gamma}(x)dx \enspace .
\end{align*}
\end{proof}

\subsection{Bounds on Density of Arete Distribution}
\label{sec:boundsondensityArete}
In this section we show that for $\Delta>0$ and setting of parameters $\alpha,\theta,\lambda$ and for large enough $\varepsilon$:
\begin{align}
\label{eq:ratio}
e^{-\varepsilon}~\le~f_A(t)/f_A(t+\Delta)~\le~e^\varepsilon,\qquad \forall t\in\mathbb{R}.
\end{align}
We remark that by monotonicity of the density of the Arete distribution, if we show (\ref{eq:ratio}) it follows that $f_A$ satisfies (\ref{eq:DPratio}):
Take any $a\in\mathbb{R}$ such that $\vert a\vert\le \Delta$ and suppose without loss of generality that $f(t)\ge f(t+a)$ (if this is not the case, substitute $t':=\vert t\vert-a$, such that $f(t')\ge f(t'+a)$). Then $e^{-\varepsilon}\le \frac{f(t)}{f(t+a)}$. We prove that $f(t+a)\ge f(t+\Delta)$ ensuring $\frac{f(t)}{f(t+a)}\le \frac{f(t)}{f(t+\Delta)}\le e^\varepsilon$, which finishes the argument: by assumption $f(t)\ge f(t+a)$ and so $\vert t\vert\le \vert t+a\vert$, further implying that $t\ge -a/2$. Hence, as $\vert a\vert \le \Delta$ we have $\vert t+\Delta\vert\ge \vert t+a\vert$ and so we conclude that $f(t+a)\ge f(t+\Delta)$ as wanted.

Throughout the section we assume that $\vert t\vert\le \vert t+\Delta\vert$ (and so $t\ge -\Delta/2$). For such $t$, $f_A(t)~\ge~f_A(t+\Delta)$ and so the first inequality in (\ref{eq:ratio}) is immediate. Hence, we put our focus toward proving the latter inequality. If $\vert t+\Delta\vert\le \vert t\vert$, the result follows by symmetry of $f_A$ (Corollary \ref{cor:symmetryArete}).

We 
start with the following lower bound on the density $f_A$:
\begin{restatable}{lemma}{lowerboundArete}
\label{lem:lowerboundArete}
Let $\DGam$ and $c_{\DGam}$ be as in Lemma \ref{lem:boundsonGammaGamma} and assume $\lambda~\le~ \Delta/\ln(2)$ for $\Delta>0$. For $-\Delta/2~\le~ t\in\mathbb{R}$
\[
f_{A}(t+\Delta)\ge f_{\Gamma}(\vert t+\Delta\vert+\DGam)c_{\DGam} c_L\qquad \text{where}\qquad c_L:=1/4.
\]
\end{restatable}
\begin{proof} By Definition \ref{def:aratedistr} and Lemma \ref{lem:boundsonGammaGamma} we have 
\begin{align*}
    f_{A}(t+\Delta)&=\int_{-\infty}^\infty f_{\Gamma-\Gamma}(t+\Delta-x)f_L(x)dx\ge \int_{-\infty}^\infty f_{\Gamma}(\vert t+\Delta-x\vert+\DGam)c_{\DGam}f_L(x)dx\\
    &\ge c_{\DGam}\int_{0}^{2(t+\Delta)} f_{\Gamma}(\vert t+\Delta-x\vert+\DGam)f_L(x)dx\ge c_{\DGam}f_{\Gamma}(t+\Delta+\DGam)\int_{0}^{2(t+\Delta)} f_L(x)dx,
\end{align*}
where we used that $f_\Gamma(\vert t+\Delta-x\vert+\DGam)\ge f_\Gamma(\vert t+\Delta\vert+\DGam)$ for $x\in(0,2(t+\Delta))$ and that by assumption $t+\Delta\ge \Delta/2$ allowing us to remove the absolute value signs.
Again using that $t+\Delta\ge \Delta/2$
\[
\int_{0}^{2(t+\Delta)} f_L(x)dx\ge \int_{0}^{\Delta} f_{L(\lambda)}(x)dx=\frac{1}{2}\int_{0}^{\Delta} f_{Exp(\lambda)}(x)dx\ge \frac{1}{4},\qquad \lambda<\Delta/\ln(2)
\]
where we noticed that on the positive reals, the density function of the Laplace distribution is $1/2$ times the density function of the Exponential distribution, and used that the median of the latter is $\ln(2)\lambda$, so that the last inequality is true as long as $\ln(2)\lambda\le \Delta$. Hence,
\[
f_{A}(t+\Delta)\ge c_{\DGam}f_{\Gamma}(t+\Delta+\DGam)1/4.
\]
Defining $c_L:=1/4$ finishes the proof.
\end{proof}

The following three lemmas are technical and give upper bounds for the ratio $f_A(t)~/~f_A(t~+~\Delta)$; first for large and small $\vert t\vert$ separately in Lemmas \ref{lem:DPratioLarget} and \ref{lem:DPratioSmallt} (i.e., for $t$ close to and far from 0, where ''close to/far from'' is quantified by a parameter $\Du$, which we will set in Lemma \ref{lem:parametersetting}). We combine these results to an upper bound for general $t$ in Lemma \ref{lem:DPratioGeneralt} (still assuming $t$ is s.t. $f_A(t)\ge f_A(t+\Delta)$) and finally choose parameters to ensure an upper bound of $e^\varepsilon$ in Lemma \ref{lem:parametersetting}, thus satisfying the second inequality of (\ref{eq:ratio}). Throughout the next three lemmas we make use the variables $\DGam,c_{\DGam}$ (from Lemma \ref{lem:boundsonGammaGamma}) and $c_L$ (from Lemma \ref{lem:lowerboundArete}), all of which will be handled in the proof of Lemma \ref{lem:parametersetting}.
\begin{restatable}{lemma}{dpratioLarget}
\label{lem:DPratioLarget}
Let $\Du>0$ be given and assume $0<\alpha\le 1$. Let $\DGam, c_{\DGam}$ be as in Lemma \ref{lem:boundsonGammaGamma} and $c_L$ as in Lemma \ref{lem:lowerboundArete}. 
Assume $1/\lambda-1/\theta\ge 1/(\Du+\Delta+\DGam)$ and $\lambda\le \Delta/\ln(2)$ for $\Delta> 0$. For $-\Delta/2<t\in\mathbb{R}$ with $\vert t\vert\ge \Du$ we have
\begin{align*}
    \frac{f_{A}(t)}{f_{A}(t+\Delta)}
    \le \frac{e^{(\DGam+\Delta)/\theta}}{c_{\DGam}}\left(\left(1+\frac{\Delta+\DGam}{\Du}\right)+\frac{c_{\Du}\Gamma(\alpha)\theta^\alpha}{2\lambda c_L}e^{\Du/\theta}(\DGam+\Delta+\Du)^{1-\alpha}\right)
\end{align*}
where $c_{\Du}:=2\int_{0}^{\Du}f_{\Gamma}(x)dx$. 
\end{restatable}
\begin{proof}
The proof can be found in Appendix \ref{app:prooflemLarget}.
\end{proof}

\begin{restatable}{lemma}{dpratioSmallt}
\label{lem:DPratioSmallt}
Let $\Du>0$ be given and assume $0<\alpha\le 1$. Let $\DGam, c_{\DGam}$ be as in Lemma \ref{lem:boundsonGammaGamma} and $c_L$ as in Lemma \ref{lem:lowerboundArete}.
Assume $1/\lambda-1/\theta\ge 1/(\Delta+\DGam)$ and $\lambda\le \Delta/\ln(2)$ for $\Delta> 0$. For $-\Delta/2<t\in\mathbb{R}$ such that $\vert t\vert\le \Du$ we have
\begin{align*}
    \frac{f_{A}(t)}{f_{A}(t+\Delta)}\le \frac{e^{(\DGam+\Delta)/\theta}}{c_Lc_{\DGam}\lambda}\left(\frac{\Du+\DGam+\Delta}{\alpha}+\Gamma(\alpha)\theta^\alpha(\DGam+\Delta)^{1-\alpha}\right)
\end{align*}
\end{restatable}
\begin{proof}
The proof can be found in Appendix \ref{app:prooflemSmallt}.
\end{proof}

The following lemma combines Lemmas \ref{lem:DPratioLarget} and \ref{lem:DPratioSmallt} to give an upper bound for general $t~>~-~\Delta/2$:
\begin{restatable}{lemma}{dpratioGeneralt}
\label{lem:DPratioGeneralt}
Let $\DGam, c_{\DGam}$ be as in Lemma \ref{lem:boundsonGammaGamma} and $c_L$ as in Lemma \ref{lem:lowerboundArete}.
Assume $0<\alpha\le 1$, $\theta\le \DGam+1$, $\lambda\le \min\{\theta/2,\Delta/\ln(2)\}$ and $\Gamma(\alpha)\le 1/\alpha$ for $\Delta> 0$. For $-\Delta/2<t\in\mathbb{R}$
\begin{align*}
     \frac{f_{A}(t)}{f_{A}(t+\Delta)}
     &\le \frac{2e^{(\DGam+\Delta)/\theta}e^\alpha(\alpha\theta+\DGam+\Delta)}{\alpha c_{\DGam}c_L\lambda}
\end{align*}
\end{restatable}
\begin{proof}
The proof can be found in Appendix \ref{app:prooflemGeneralt}.
\end{proof}
\begin{note}
The fact that $\Gamma(\alpha) \leq 1/\alpha$ for $0<\alpha\le 1$ follows from Euler's definition of the Gamma function, 
$$\Gamma(\alpha) = \frac{1}{\alpha} \prod_{n=1}^{\infty} \frac{(1+\frac{1}{n})^\alpha}{1+\frac{\alpha}{n}}, $$
since $(1+\frac{1}{n})^\alpha \leq 1+\frac{\alpha}{n}$ for any $n>0$ and $0<\alpha\le 1$.
\end{note}

We finally choose parameters $\alpha,\theta,\lambda$ ensuring that the ratio $e^{-\varepsilon}\le f(t)/f(t+\Delta)\le e^\varepsilon$ for $\varepsilon$ large enough:
\begin{restatable}{lemma}{paramsetting}
\label{lem:parametersetting}
Suppose $\varepsilon\ge 20+4\ln(\Delta)$ for $\Delta\ge 2/e$. Let $\alpha=e^{-\varepsilon/4}, \theta=\frac{4\Delta}{\varepsilon}$ and  $\lambda=e^{-\varepsilon/4}$.
Then for $t\in\mathbb{R}$
\[
e^{-\varepsilon}\le \frac{f_{A}(t)}{f_{A}(t+\Delta)}\le e^\varepsilon.
\]
\end{restatable}
\begin{proof}
The proof can be found in Appendix \ref{app:proofparamsetting}.
\end{proof}

\subsection{Putting Things Together}
\label{sec:puttingThingsTogether}
We restate the lemma here for convenience:
\main*
\begin{proof}
The first bullet with the choice of parameters $\alpha=e^{-\varepsilon/4}, \theta=4\Delta/\varepsilon$ and $\lambda=e^{-\varepsilon/4}$ follow from Lemma \ref{lem:parametersetting}
and  monotonicity of $f_A$ (Lemma \ref{lem:mainAreteProperties}), as described at the beginning of Section~\ref{sec:boundsondensityArete}. 
The second bullet also follows from Lemma \ref{lem:parametersetting}, as the expected error of a random variable $Z~=~X~+~Y$, where $X\sim\Gamma-\Gamma(\alpha,\theta)$ and $Y\sim Laplace(\lambda)$, i.e., $Z\sim Arete(\alpha,\theta,\lambda)$,
is 
\[
\E[\vert Z\vert]=\E[\vert X_1-X_2+Y\vert]\le 2\E[ X_1]+\E[\vert Y\vert]=2\alpha\theta+\lambda=\frac{8\Delta e^{-\varepsilon/4}}{\varepsilon}+e^{-\varepsilon/4}=O\left(\frac{\Delta}{\varepsilon}e^{-\varepsilon/4}\right)
\]
where $X_1,X_2\sim\Gamma(\alpha,\theta)$, and similarly, by independence
\begin{align*}
    \Var[Z]&=\Var[X+Y]=\Var[X_1-X_2+Y]=2\Var[X_1]+\Var[ Y]=2\alpha\theta^2+2\lambda^2\\&=2\frac{16\Delta^2e^{-\varepsilon/4}}{\varepsilon^2}+2e^{-\varepsilon/2}=O\left(\frac{\Delta^2}{\varepsilon^2}e^{-\varepsilon/4}\right)
\end{align*}
with our choice of parameters. Finally, for $\varepsilon\ge 1/\sqrt{2}$ (which is significantly smaller than the values of $\varepsilon$ that we are interested in), we may simplify to 
\[
\E[\vert Z\vert]=O\left(\Delta e^{-\varepsilon/4}\right),\qquad \Var[Z]=O\left(\Delta^2e^{-\varepsilon/4}\right)
\]
thus finishing the proof.
\end{proof}

\begin{figure}[t]
\centering
\subfigure[Densities for $\varepsilon=6$]{
\includegraphics[width=.45\textwidth]{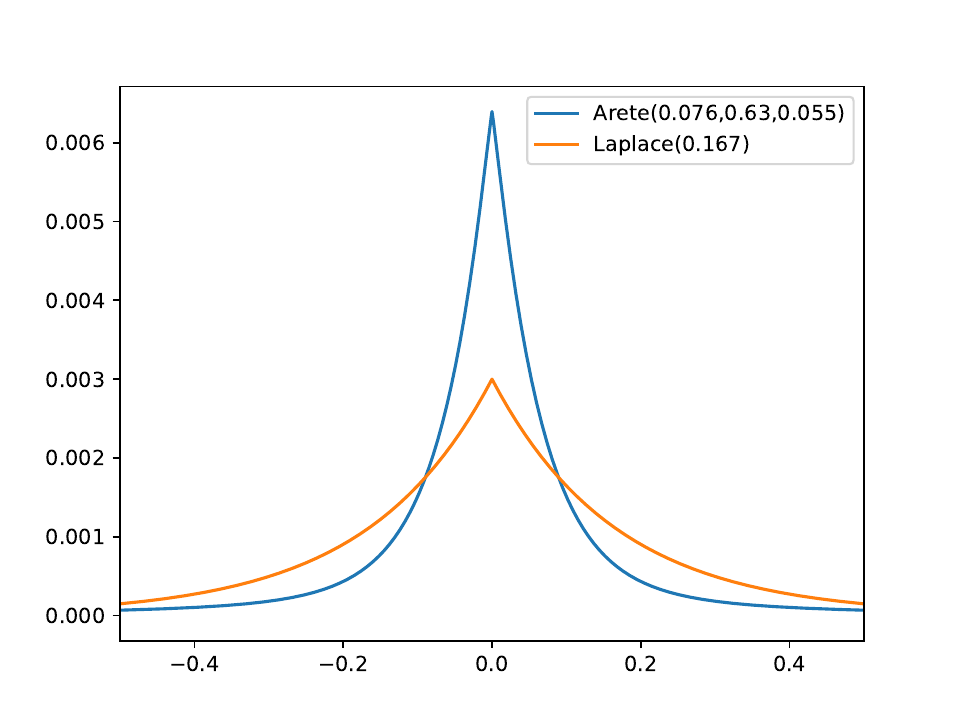}}
\subfigure[Densities for $\varepsilon=8$]{
\includegraphics[width=.45\textwidth]{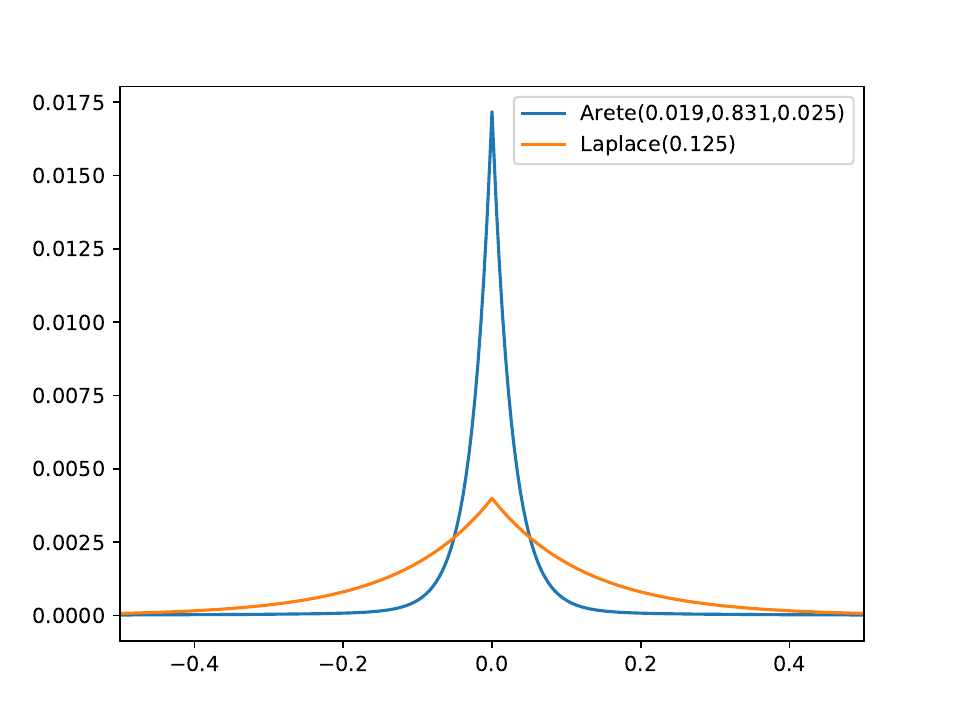}}
\caption{\footnotesize{Density functions for Arete distributions that empirically yield $\varepsilon$-DP with $\varepsilon = 6$ and $\varepsilon = 8$, respectively. The density functions were approximated by rounding the constituent $\Gamma$ and Laplace distribution values to a multiple of $0.001$ and computing the discrete convolution. Parameters were found using a local search heuristic. For comparison, Laplace distributions with the same privacy guarantee have been included and are clearly less concentrated around zero.}}\label{fig:empirical_arete}
\end{figure}

\section{Conclusion}
\label{sec:conclusion}
In this work we have seen a new noise distribution, the Arete distribution, which has a continuous density function, is symmetric around zero, monotonely decreasing for $t>0$, infinitely divisible and has expected absolute value and variance exponentially decreasing in $\varepsilon$. The Arete distribution yields an $\varepsilon$-differentially private mechanism, the Arete mechanism, which is an \emph{infinitely divisible} alternative to the Staircase mechanism \citep{GengV16} with a \emph{continuous} density function. The Arete mechanism achieves error comparable to the Staircase mechanism and outperforms the Laplace mechanism \citep{DworkMNS06} in terms of absolute error and variance in the low privacy regime (for large $\varepsilon$). 

Simulations suggest that the constant factors of the Arete mechanism, with parameters chosen as we have described, can be improved (see Figure~\ref{fig:empirical_arete}). We leave open the question of finding an optimal (up to lower order terms), infinitely divisible error distribution for differential privacy in the low privacy regime.

\acks{We thank Thomas Steinke for useful feedback on a previous version of this paper, and the anonymous reviewers for their constructive suggestions. This work was done while both authors were affiliated with BARC, supported by the VILLUM Foundation grant 16582.}

\bibliography{bibl}

\appendix

\section{Applications}
\label{sec:applications}
In differential privacy, two models are prevalent: the central model and the local model. In the \emph{central} model of differential privacy \citep{DworkKMMN06} all data is held by a single trusted unit who makes the result of a query differentially private before releasing it. This is often done by adding noise to the query result. The central model usually has a very high level of accuracy, but requires a high level of trust. Often, data is split among many players, we refer to these as data owners, and a trusted central unit is not available. This setting is commonly known as the \emph{local} model of differential privacy \citep{KasiviswanathanLNRS11,warner1965randomized,DuchiJW13}. In this model, each data owner must ensure privacy for their own data, and so applies a differentially private mechanism locally, which is then forwarded to an analyst who combines all reports to compute an approximate answer to the query. For many queries, the overall error in the local model grows rather quickly as a function of the number of players, significantly limiting utility.
For example, while we can achieve constant error in the central model \citep{DworkMNS06}, a count query requires $O(\sqrt{n})$ error for the same level of privacy as in the central model, where $n$ is the number of players \citep{CheuSUZZ19}. %
The local model is often attractive for data collection as the collecting organizations are not liable for storing sensitive user data in this model -- a few examples of deployment include Google's RAPPOR \citep{ErlingssonPK14}, Apple (several features such as Lookup Hints, Emoji suggestion etc.) \citep{AppleLDP} and Microsoft Telemetry \citep{DingKY17}.

In order to bridge this trust/utility gap, we may imitate the trusted unit from the centralized setting with cryptographic primitives \citep{WaghHMM21}, allowing for differentially private implementations with better utility than in the local model while having lower trust assumptions than in the centralized model. Cryptographic primitives ensure that all parties learn only the output of the computation, while differential privacy further bounds the information leakage from this output, and so the combination gives very strong guarantees.
We limit our discussion to the problem of computing the sum of real inputs, which is a basic building block in many other applications. If we can divide the noise among all players we can obtain the same accuracy in a distributed setting as in the central model without the assumption of a trusted aggregator. %
Luckily, we can divide the noise between the players if the noise distribution $\mathcal{D}$ is infinitely divisible, %
and so the Arete distribution can be applied in this model. 

We discuss differential privacy implementations with two cryptographic primitives: Secure Multiparty Aggregation and Anonymous Communication but note that such implementations come with assumptions about the computational power of the analyst, which are accepted by the security community, but limit the privacy guarantee to computational differential privacy \citep{WaghHMM21}.

\subsubsection*{Secure Multiparty Aggregation}
The cryptographic primitive secure multiparty Aggregation, rooted in the work of Yao \citep{Yao82b}, has often been combined with differential privacy to solve the problem of private real summation, see for example \citep{ShiCRCS11, BonawitzIKMMPRS17, ChanSS12_2}.
\cite{GoryczkaX17} 
give a comparative study of several protocols for private summation in a distributed setting. These protocols combine common approaches for achieving security (secret sharing, homomorphic encryption and perturbation-based) while each party adds noise shares whose sum follows the Laplace distribution before sharing their data, in order to ensure differential privacy. Continuing their line of work, we may exchange the Laplace noise in \citep{GoryczkaX17} with Arete distributed noise to achieve $\varepsilon$-differentially private protocols with error exponentially small in $\varepsilon$.

\subsubsection*{Anonymous Communication}
Another line of work that has received a lot of attention over the past few years is the \emph{shuffle} model of differential privacy \citep{BittauEMMRLRKTS17, CheuSUZZ19, ErlingssonFMRTT19}. %
Along with Google's Prochlo framework, \cite{BittauEMMRLRKTS17} introduced the ESA (Encode Shuffle Analyze) framework where each data owner encodes their data before releasing it to a shuffler. The shuffler randomly permutes the encoded inputs and releases the (private) permuted set of data to an (untrusted) analyst, who then performs statistical analysis on the encoded, shuffled data. %
For recent work on the problem of summation in the shuffle model and a discussion of error/privacy-tradeoff, we refer to for example \citep{BalleBGN19, BalleBGN20, GhaziGKMPV20, GhaziKMPS21, GhaziMPV20}.
\cite{GhaziKMPS21} propose an $(\varepsilon,\delta)$-differentially private protocol for summation in the shuffle model for summing reals or integers where each user sends expected $1+o(1)$ messages. The protocol adds discrete Laplace noise (also sometimes called Geometrically distributed noise) and achieves error arbitrarily close to that of the Laplace mechanism (applied in the central model), but do not address the problem of achieving error exponentially decreasing in $\varepsilon$ in the shuffle model. The Arete distribution solves this open problem: as it is infinitely divisible, simply exchange the discrete Laplace noise (the ``central'' noise distribution in the protocol) with the Arete distribution. This yields:
\begin{corollary}[Differentially private aggregation in the shuffle model.] Let $n$ be a positive integer, and let $\varepsilon$, $\delta$ be positive real numbers with $\varepsilon = O(\ln n)$. There is an $(\varepsilon,\delta)$-differentially private aggregation protocol in the shuffle model for inputs in $[0,1]$ having absolute error $\frac{1}{e^{\Omega(\varepsilon)} - 1}$ in expectation, using $O\left(1+\frac{\log(1/\delta)}{\log n}\right)$ messages per party, each consisting of $O(\log n)$ bits.
\end{corollary}

{\bf Note.} By the post-processing property of differential privacy, we still achieve differential privacy if more than $n$ players participate, and so we only need to choose the noise shares based on a lower bound on the number of players in order to ensure differential privacy. Hence, it is not strictly necessary to know the exact number of players in advance.

\section{Basic definitions}
\label{sec:basics}

\subsection{Probability Distributions}
In this section we state the definitions and basic facts that we need to analyze the Arete distribution. References to further information can be found in~\citep{GoryczkaX17}.

\begin{definition}[Infinite Divisibility]
\label{def:infdiv}
A distribution $\mathcal{D}$ is infinitely divisible if, for any random variable $X$ with distribution $\mathcal{D}$, then for every positive integer $n$ there exist $n$ i.i.d. random variables $X_1,...,X_n$ such that $\sum_{i=1}^n X_i$ has the same distribution as $X$. The random variables $X_i$ need not have distribution $\mathcal{D}$.
\end{definition}

We recall the definitions of the distributions that we use to define the Arete distribution and give a formal definition of the latter. Whenever the parameters are implicit we leave them out and simply write $f_{\Gamma}$, $f_L$, $f_{\Gamma-\Gamma}$ and $f_A$ for the densities of the $\Gamma$, Laplace, $\Gamma-\Gamma$ and Arete distributions, resp.
\begin{definition}[The $\Gamma$ Distribution]
\label{def:gammadistr}
A random variable $X$ has Gamma distribution with shape parameter $\alpha>0$ and scale parameter $\theta>0$, denoted $X\sim\Gamma(\alpha,\theta)$, if its density function is
\[
f_{\Gamma(\alpha,\theta)}(t)=\frac{e^{-t/\theta}t^{\alpha-1}}{\Gamma(\alpha)\theta^\alpha},\quad t>0.
\]
In the special case $\alpha=1$, the random variable $X$ has Exponential distribution with parameter $\theta$.
\end{definition}

The $\Gamma$-distribution is infinitely divisible: For $n$ independent random variables $X_i\sim \Gamma(\alpha_i,\theta)$, we have $X=\sum_{i=1}^nX_i\sim\Gamma\left(\sum_{i=1}^n \alpha_i,\theta\right)$.
Furthermore, for $X\sim\Gamma(\alpha,\theta)$ we have
$\E[X]=\alpha\theta$ and $\Var[X]=\alpha\theta^2$.

\begin{definition}[The Laplace Distribution]
\label{def:laplacedistr}
A random variable $X$ has Laplace distribution with location parameter $\mu$ and scale parameter $\lambda>0$, denoted $X\sim Laplace(\mu,\lambda)$, if its density function is
\[
f_{L(\mu,\lambda)}(t)=\frac{e^{-\vert t-\mu\vert/\lambda}}{2\lambda},\quad t\in\mathbb{R}.
\]
If $\mu=0$ we just write $Laplace(\lambda)$. 
\end{definition}

If $X\sim Laplace(\lambda)$, then $\vert X\vert\sim Exp(\lambda)$ and $\E[X]=0$ while $\E[\vert X\vert]=\lambda$. Similarly, $\Var[ X]=2\lambda^2$ while $\Var[ \vert X\vert]=\lambda^2$.\\
The Laplace distribution is infinitely divisible: For $2n$ independent random variables $X_i,~Y_i~\sim~ \Gamma(1/n,\lambda)$, we have $X=\sum_{i=1}^n(\mu/n+X_i-Y_i)\sim Laplace\left(\mu, \lambda\right)$.

\subsection{Differential Privacy}
\label{sec:DP}
Informally, differential privacy promises that an analyst cannot, given a query answer, decide whether the underlying data contains a specific data record or not, and so differential privacy relies on the notion of \emph{neighboring inputs}: datasets $x,y\in\mathcal{X}^d$ are neighbors, denoted $x\doteq y$, if they differ by one data record.
The \emph{sensitivity} of a query quantifies how much the output of the query can differ for neighboring inputs, and so describes how much difference the added noise needs to hide.
\begin{definition}[Sensitivity \citep{DworkMNS06}]
\label{def:sensitivity}
For a real-valued query $q:\mathcal{X}^d\rightarrow\mathbb{R}$, the sensitivity of $q$ is defined as\\ $\max_{x,y\in\mathcal{X}^d:\ x\doteq y}\vert q(x)-q(y)\vert$.
\end{definition}

\begin{definition}[Differential Privacy \citep{DworkMNS06,DworkKMMN06}]
Let $\mathcal{M}$ be a randomized mechanism. For privacy parameters $\varepsilon,\delta\ge 0$, %
we say that $\mathcal{M}$ is $(\varepsilon,\delta)$-differentially private if, for any neighboring inputs $x, y\in \mathcal{X}^d$ and all $S\in Range(\mathcal{M})$ we have 
\[
\Pr[\mathcal{M}(x)\in S]\le e^\varepsilon\Pr[\mathcal{M}(y)\in S]+\delta.
\]
If $\delta=0$ we say that $\mathcal{M}$ is $\varepsilon$-differentially private. 
\end{definition}
For more details about differential privacy, we refer the reader to \citep{Dwork08,DworkR14,Vadhan17}.

\begin{lemma}[The Laplace Mechanism \citep{DworkMNS06}]
\label{lem:lapmech}
For real-valued query $q:\mathcal{X}^d\rightarrow\mathbb{R}$ %
and input $x\in\mathcal{X}^d$, the Laplace mechanism outputs $q(x)+X$ where $X\sim Lap(\lambda)$. If $\Delta$ is the sensitivity of $q$, the Laplace mechanism with parameter $\lambda=\Delta/\varepsilon$ is $\varepsilon$-differentially private.
\end{lemma}

\begin{lemma}[The Staircase Mechanism \citep{GengV16}]
\label{lem:staircasemech}
Let $q:\mathbb{R}\rightarrow\mathbb{R}$ be a real-valued query with sensitivity $\Delta$. Let random variable $X\sim SC(\gamma,\Delta)$ have Staircase distribution with parameters $\gamma\in[0,1]$ and $\Delta>0$ such that the density of $X$ is
\[
f_{SC}(t)=\begin{cases}
a(\gamma),\qquad\qquad\qquad\quad\  t\in[0,\gamma\Delta)\\
e^{-\varepsilon}a(\gamma),\qquad\qquad\qquad t\in[\gamma\Delta,\Delta)\\
e^{-k\varepsilon}f_{SC}(t-k\Delta),\qquad t\in[k\Delta, (k+1)\Delta),\ k\in\mathbb{N}\\
f_{SC}(-t),\qquad\qquad\quad\ \ t<0
\end{cases}
\]
where $a(\gamma)=\frac{1-e^{-\varepsilon}}{2\Delta(\gamma+e^{-\varepsilon}(1-\gamma))}$ is a normalization factor.
Then for input $x\in\mathbb{R}$, the Staircase mechanism which outputs $q(x)+X$ where $X\sim SC(\gamma,\Delta)$
is $\varepsilon$-differentially private. %
\end{lemma}
For optimal parameter $\gamma$, the Staircase mechanism achieves expected absolute error $\Theta(\Delta e^{-\varepsilon/2})$ and variance $\Theta(\Delta^2e^{-2\varepsilon/3})$. We remark that the $\gamma$ optimizing for expected magnitude of the noise is not the same as the $\gamma$ optimizing for variance.

\section{Omitted Proofs for Technical Results}
\label{app:omittedproofs}
\subsection*{Supporting lemmas}
\begin{lemma}
\label{lem:maximizingfraction}
Let $\DGam$ be as in Lemma \ref{lem:boundsonGammaGamma} and $\Delta>0$.
Assume $0<\alpha<1$, $1/\lambda-1/\theta\ge \frac{1}{\kappa+\Delta+\DGam}.$
Then $\forall t\ge \kappa\ge 0$
\[
\frac{(\DGam+\Delta+ t)^{1-\alpha}}{e^{ t(1/\lambda-1/\theta)}}\le \frac{(\DGam+\Delta+ \kappa)^{1-\alpha}}{e^{\kappa(1/\lambda-1/\theta)}}.
\]
\end{lemma}
\begin{proof}
The function
\[
g(t)=\frac{(\DGam+\Delta+ t)^{1-\alpha}}{e^{ t(1/\lambda-1/\theta)}}
\]
maximized for
\[
t^*=\frac{1-\alpha-(1/\lambda-1/\theta)(\DGam+\Delta)}{1/\lambda-1/\theta}=\frac{1-\alpha}{1/\lambda-1/\theta}-(\DGam+\Delta),\qquad 1/\lambda-1/\theta>0,\quad 0<\alpha<1,
\]
and monotonely decreasing for $t\ge t^*$.
By assumption
\[
\kappa\ge \frac{1}{1/\lambda-1/\theta}-(\Delta+\DGam)\ge t^*.
\]
and so $g(\kappa)\le g(t^*)$. Furthermore, for all $t\ge \kappa$
\[
g(t)=\frac{(\DGam+\Delta+ t)^{1-\alpha}}{e^{ t(1/\lambda-1/\theta)}}\le \frac{(\DGam+\Delta+ \kappa)^{1-\alpha}}{e^{\kappa(1/\lambda-1/\theta)}}=g(\kappa).
\]
\end{proof}

\subsection{Proof of Lemma \ref{lem:DPratioLarget}}
\label{app:prooflemLarget}
\dpratioLarget*
\begin{proof}[Proof of Lemma \ref{lem:DPratioLarget}] Suppose $\Du\le \vert t\vert$. By Lemma \ref{lem:boundsonGammaGamma}
\begin{align}
\label{eq:upperboundnumintegral}
    \frac{f_{A}(t)}{f_{A}(t+\Delta)}&=\frac{\int_{-\infty}^{\infty}f_{\Gamma-\Gamma}(t-x)f_{L}(x)dx}{\int_{-\infty}^{\infty}f_{\Gamma-\Gamma}(t+\Delta-x)f_{L}(x)dx}\le \frac{\int_{-\infty}^{\infty}f_{\Gamma}(\vert t-x\vert)f_{L}(x)dx}{\int_{-\infty}^{\infty}f_{\Gamma}(\vert t+\Delta-x\vert+\DGam)c_{\DGam}f_{L}(x)dx}.
\end{align}
Note that $\vert t+\Delta-x\vert\le \vert t-x\vert +\Delta$ and 
\begin{align}
\label{eq:rewritetminusx}
\vert t-x\vert^{\alpha-1}=(\vert t-x\vert+\DGam+\Delta)^{\alpha-1}\left(1+\frac{\Delta+\DGam}{\vert t-x\vert}\right)^{1-\alpha}.
\end{align}
So filling in the density $f_\Gamma$ and applying (\ref{eq:rewritetminusx}), we can write (\ref{eq:upperboundnumintegral}) as
\begin{align*}
    \frac{f_{A}(t)}{f_{A}(t+\Delta)}&\le \frac{\int_{-\infty}^{\infty}\frac{1}{\Gamma(\alpha)\theta^\alpha}e^{-\vert t-x\vert/\theta}(\vert t-x\vert+\DGam+\Delta)^{\alpha-1}\left(1+\frac{\Delta+\DGam}{\vert t-x\vert}\right)^{1-\alpha}f_{L}(x)dx}{c_{\DGam}\int_{-\infty}^{\infty}f_{\Gamma}(\vert t-x\vert+\Delta+\DGam)f_{L}(x)dx}.
\end{align*}
Since
$\left(1+\frac{\Delta+\DGam}{\vert t-x\vert}\right)^{1-\alpha}$ is maximized for $x\rightarrow t$, we can bound this term as long as $x$ is not too close to $t$. Hence, rewind to equation (\ref{eq:upperboundnumintegral}) and treat the cases where $x$ is far from $t$ and $x$ is close to $t$ separately by splitting the numerator from (\ref{eq:upperboundnumintegral}) at the intervals $x\in(-\infty, t-\Du)\cup(t+\Du,\infty)$ and $x\in(t-\Du,t+\Du)$ (these intervals are well-defined since $\Du>0$):
\begin{align*}
    \frac{f_{A}(t)}{f_{A}(t+\Delta)}&\le  \frac{\int_{-\infty}^{t-\Du}f_{\Gamma}(\vert t-x\vert)f_{L}(x)dx+\int_{t+\Du}^\infty f_{\Gamma}(\vert t-x\vert)f_{L}(x)dx}{c_{\DGam}\int_{-\infty}^{\infty}f_{\Gamma}(\vert t-x\vert+\Delta+\DGam)f_{L}(x)dx}\\
	& \qquad+ \frac{\int_{ t-\Du}^{t+\Du}f_{\Gamma}(\vert t-x\vert)f_{L}(x)dx}{c_{\DGam}\int_{-\infty}^{\infty}f_{\Gamma}(\vert t+\Delta-x\vert+\DGam)f_{L}(x)dx}\\
    &=\frac{\int_{-\infty}^{t-\Du}e^{-\vert t-x\vert/\theta}(\vert t-x\vert+\DGam+\Delta)^{\alpha-1}\left(1+\frac{\Delta+\DGam}{\vert t-x\vert}\right)^{1-\alpha}f_{L}(x)dx}{c_{\DGam}\int_{-\infty}^{\infty}e^{-(\vert t-x\vert+\Delta+\DGam)/\theta}(\vert t-x\vert+\Delta+\DGam)^{\alpha-1}f_{L}(x)dx}\\
    &\qquad+\frac{\int_{ t+\Du}^\infty e^{-\vert t-x\vert/\theta}(\vert t-x\vert+\DGam+\Delta)^{\alpha-1}\left(1+\frac{\Delta+\DGam}{\vert t-x\vert}\right)^{1-\alpha}f_{L}(x)dx}{c_{\DGam}\int_{-\infty}^{\infty}e^{-(\vert t-x\vert+\Delta+\DGam)/\theta}(\vert t-x\vert+\Delta+\DGam)^{\alpha-1}f_{L}(x)dx}\\
    &\qquad+\frac{\int_{t-\Du}^{ t+\Du}f_{\Gamma}(\vert t-x\vert)f_{L}(x)dx}{c_{\DGam}\int_{-\infty}^{\infty}f_{\Gamma}(\vert t+\Delta-x\vert+\DGam)f_{L}(x)dx}
\end{align*}
where we in the last step again filled in the density function $f_\Gamma$ and applied (\ref{eq:rewritetminusx}) in the first two terms and left the last term as it was. Note that the constant $\Gamma(\alpha)\theta^\alpha$ from the density $f_\Gamma$ cancels out in the fraction.

Now (still leaving the last term alone), for the first two terms upper bound the factor \[\left(1+\frac{\Delta+\DGam}{\vert t-x\vert}\right)^{1-\alpha}\le \left(1+\frac{\Delta+\DGam}{\Du}\right)^{1-\alpha}\] and pull out $e^{(\Delta+\DGam)/\theta}$ from the denominator, to see that
\begin{align*}
    &\frac{f_{A}(t)}{f_{A}(t+\Delta)}\\
    &\le e^{(\Delta+\DGam)/\theta}\left(1+\tfrac{\Delta+\DGam}{\Du}\right)^{1-\alpha}\frac{\left(\int_{-\infty}^{t-\Du}\frac{(\vert t-x\vert+\DGam+\Delta)^{\alpha-1}}{e^{\vert t-x\vert/\theta}}f_{L}(x)dx+\int_{ t+\Du}^\infty \frac{(\vert t-x\vert+\DGam+\Delta)^{\alpha-1}}{e^{\vert t-x\vert/\theta}}f_{L}(x)dx\right)}{c_{\DGam}\int_{-\infty}^{\infty}\frac{(\vert t-x\vert+\DGam+\Delta)^{\alpha-1}}{e^{\vert t-x\vert/\theta}}f_{L}(x)dx}\\
    &\qquad+\frac{\int_{t-\Du}^{ t+\Du}f_{\Gamma}(\vert t-x\vert)f_{L}(x)dx}{c_{\DGam}\int_{-\infty}^{\infty}f_{\Gamma}(\vert t+\Delta-x\vert+\DGam)f_{L}(x)dx}\\
    &\le \frac{1}{c_{\DGam}}\left(e^{(\Delta+\DGam)/\theta}\left(1+\frac{\Delta+\DGam}{\Du}\right)+\frac{\int_{t-\Du}^{ t+\Du}f_{\Gamma}(\vert t-x\vert)f_{L}(x)dx}{\int_{-\infty}^{\infty}f_{\Gamma}(\vert t+\Delta-x\vert+\DGam)f_{L}(x)dx}\right)
\end{align*}

\noindent
where in the last step we upper bounded the fraction in the first term by $1/c_{\DGam}$ and removed the $(1-\alpha)$-exponent for simpler notation.
The following claim handles the last term and finishes the proof:
\begin{claim}
Let $\Du>0$ be given. Let $\DGam, c_{\DGam}$ be as in Lemma \ref{lem:boundsonGammaGamma} and $c_L$ as in Lemma \ref{lem:lowerboundArete}.
Assume $\vert t\vert \ge \Du$, $\vert t\vert\le \vert t+\Delta\vert$, $\lambda\le \Delta/\ln(2)$ and $1/\lambda-1/\theta\ge 1/(\DGam+\Delta+\Du)$. Then
\begin{align*}
    \frac{\int_{t-\Du}^{ t+\Du}f_{\Gamma}(\vert t-x\vert)f_L(x)dx}{\int_{-\infty}^{\infty}f_{\Gamma}(\DGam+\vert t+1-x\vert)f_L(x)dx}\le\frac{c_{\Du}\Gamma(\alpha)\theta^\alpha}{2\lambda c_L}e^{(\DGam+\Du+\Delta)/\theta}(\DGam+\Delta+\Du)^{1-\alpha}
\end{align*}
where 
\[
c_{\Du}:=2\int_{0}^{\Du}f_{\Gamma}(x)dx.
\]
\end{claim}
\end{proof}

\begin{proof}[Proof of Claim]
As $\vert t\vert>\Du$,  $\{0\}\not\in(t-\Du,t+\Du)$ and so $f_L(x)$ is maximal at
\[x~=~\min\{\vert t-\Du\vert,\vert t+\Du\vert\}~=~\min\{\vert t\vert+\Du,\vert t\vert-\Du\}~=~\vert t\vert-\Du.\] Hence
\begin{align}
\label{eq:upperboundclaim}
     \int_{t-\Du}^{ t+\Du}f_{\Gamma}(\vert t-x\vert)f_L(x)dx
     \le f_{L}( \vert t\vert-\Du)\int_{ t-\Du}^{ t+\Du}f_{\Gamma}(\vert t-x\vert)dx=f_{L}(\vert t\vert-\Du)c_{\Du}
\end{align}
where we defined
\[
c_{\Du}:=\int_{ t-\Du}^{ t+\Du}f_{\Gamma}(\vert t-x\vert)dx=\int_{-\Du}^{\Du}f_{\Gamma}(\vert x\vert)dx=2\int_{0}^{\Du}f_{\Gamma}(x)dx.
\]
Now, consider 
\[
\frac{\int_{t-\Du}^{ t+\Du}f_{\Gamma}(\vert t-x\vert)f_L(x)dx}{\int_{-\infty}^{\infty}f_{\Gamma}(\DGam+\vert t+\Delta-x\vert)f_L(x)dx}.
\]
Recalling the assumptions $\vert t\vert\le \vert t+\Delta\vert$ and $\lambda\le \Delta/\ln(2)$, apply Lemma \ref{lem:lowerboundArete} in the denominator and (\ref{eq:upperboundclaim}) in the numerator, we get
\begin{align*}
    \frac{\int_{t-\Du}^{ t+\Du}f_{\Gamma}(\vert t-x\vert)f_L(x)dx}{\int_{-\infty}^{\infty}f_{\Gamma}(\DGam+\vert t+\Delta-x\vert)f_L(x)dx}&\le \frac{f_{L}( \vert t\vert-\Du)c_{\Du}}{f_{\Gamma}(\DGam+\vert t+\Delta\vert)c_{L}}\stackrel{(\ast)}\le \frac{c_{\Du}e^{-\vert \vert t\vert-\Du\vert/\lambda}\Gamma(\alpha)\theta^\alpha}{2\lambda c_Le^{-(\DGam+\Delta+ \vert t\vert)/\theta}(\DGam+\Delta+\vert t\vert)^{\alpha-1}}\\
    &\le \frac{c_{\Du}\Gamma(\alpha)\theta^\alpha}{2\lambda c_L}e^{\Du/\lambda}e^{(\DGam+\Delta)/\theta}\frac{(\DGam+\Delta+ \vert t\vert)^{1-\alpha}}{e^{ \vert t\vert(1/\lambda-1/\theta)}}
\end{align*}
where we at $(\ast)$ filled in the density functions $f_\Gamma$ and $f_L$ and used that $\vert t+\Delta\vert\le \vert t\vert +\Delta$. In the last step, recall $\vert t\vert>\Du$, so $\vert t\vert-\Du>0$. Applying Lemma \ref{lem:maximizingfraction} (recall $0<\alpha<1$ and the assumption $1/\lambda-1/\theta\ge 1/(\Delta+\DGam+\Du)$) with $\kappa=\Du$, we get
\begin{align*}
    \frac{\int_{t-\Du}^{ t+\Du}f_{\Gamma}(\vert t-x\vert)f_L(x)dx}{\int_{-\infty}^{\infty}f_{\Gamma}(\DGam+\vert t+\Delta-x\vert)f_L(x)dx}&\le \frac{c_{\Du}\Gamma(\alpha)\theta^\alpha}{2\lambda c_L}e^{\Du/\lambda}e^{(\DGam+\Delta)/\theta}\frac{(\DGam+\Delta+\Du)^{1-\alpha}}{e^{\Du(1/\lambda-1/\theta)}}\\&=\frac{c_{\Du}\Gamma(\alpha)\theta^\alpha}{2\lambda c_L}e^{(\DGam+\Du+\Delta)/\theta}(\DGam+\Delta+\Du)^{1-\alpha}.
\end{align*}
\end{proof}

\subsection{Proof of Lemma \ref{lem:DPratioSmallt}}
\label{app:prooflemSmallt}
\dpratioSmallt*
\begin{proof}[Proof of Lemma \ref{lem:DPratioSmallt}]
Suppose $\vert t\vert <\Du$.
By Lemmas \ref{lem:integralsymmetricfunctions} and \ref{lem:boundsonGammaGamma} we have
\[
f_A(t)=f_A(\vert t\vert)=\int_{-\infty}^{\infty}f_{\Gamma-\Gamma}(x)f_L(\vert t\vert-x)dx\le \int_{-\infty}^{\infty}f_{\Gamma}(\vert x\vert)f_L( \vert t\vert-x)dx.
\]
Note that $f_L(\vert t\vert-x)\le f_L(t)$ when $\vert \vert t\vert -x\vert\ge \vert t\vert$, which is satisfied whenever $ x\not\in(0,2\vert t\vert)$.
\begin{align*}
    f_A(t)&=f_A(\vert t\vert)=\int_{-\infty}^{\infty}f_{\Gamma-\Gamma}(x)f_L(\vert t\vert-x)dx\le \int_{-\infty}^{\infty}f_{\Gamma}(\vert x\vert)f_L( \vert t\vert-x)dx\\&=\int_{-\infty}^{0}f_{\Gamma}(\vert x\vert)f_L( \vert t\vert-x)dx+\int_{0}^{2 \vert t\vert}f_{\Gamma}(\vert x\vert)f_L(\vert t\vert-x)dx+\int_{2 \vert t\vert}^{\infty}f_{\Gamma}(\vert x\vert)f_L(\vert t\vert-x)dx\\&\le f_L(t)\left(\int_{-\infty}^{0}f_{\Gamma}(\vert x\vert)dx+\int_{2 \vert t\vert}^{\infty}f_{\Gamma}(\vert x\vert)dx\right)+\int_{0}^{2 \vert t\vert}f_{\Gamma}(\vert x\vert)f_L(\vert t\vert-x)dx\\&\stackrel{(\ast)}\le 2f_L(t)+2\int_{0}^{\vert t\vert}f_{\Gamma}(\vert x\vert)f_L( \vert t\vert-x)dx\\&=2f_L(t)+\frac{2}{\Gamma(\alpha)\theta^\alpha2\lambda}\int_{0}^{\vert t\vert}e^{- x/\theta}x^{\alpha-1}e^{-\vert \vert t\vert-x\vert/\lambda}dx
    \\&=2f_L(t)+\frac{1}{\Gamma(\alpha)\theta^\alpha\lambda}e^{- \vert t\vert/\lambda}\int_{0}^{\vert t\vert}e^{x(1/\lambda- 1/\theta)}x^{\alpha-1}dx\\&\le 2f_L(t)+\frac{1}{\Gamma(\alpha)\theta^\alpha\lambda}e^{- \vert t\vert/\theta}\int_{0}^{\vert t\vert}x^{\alpha-1}dx\\&= 2f_L(t)+\frac{1}{\Gamma(\alpha)\theta^\alpha\lambda}e^{-\vert t\vert/\theta}\frac{\vert t\vert^\alpha}{\alpha}
\end{align*}
At $(\ast)$ we use that $f_\Gamma(\vert x\vert)$ is smaller on $(t,2t)$ than on $(0,t)$. In the last step we used that
\[
\int_{0}^\kappa x^{n}dx=\frac{\kappa^{n+1}}{n+1},\qquad n\neq -1.
\]

Recalling the assumptions $\vert t\vert\le \vert t+\Delta\vert$ and $\lambda\le \Delta/\ln(2)$, apply Lemma \ref{lem:lowerboundArete} to get 
\begin{align*}
    \frac{f_{A}(t)}{f_{A}(t+\Delta)}
    &\le \frac{2f_L(t)+\frac{1}{\Gamma(\alpha)\theta^\alpha\lambda}e^{-\vert t\vert/\theta}\frac{ \vert t\vert^\alpha}{\alpha}}{c_Lc_{\DGam}f_{\Gamma}(\DGam+\vert t+\Delta\vert)}=\tfrac{1}{c_Lc_{\DGam}}\left(\frac{2f_L(t)}{f_{\Gamma}(\DGam+\vert t+\Delta\vert)}+\tfrac{1}{\Gamma(\alpha)\theta^\alpha\lambda}\frac{e^{- \vert t\vert/\theta}\frac{ \vert t\vert^\alpha}{\alpha}}{f_{\Gamma}(\DGam+\vert t+\Delta\vert)}\right)\\
    &\stackrel{(\ast)}=\tfrac{1}{c_Lc_{\DGam}}\left(\frac{2e^{-\vert t\vert/\lambda}\Gamma(\alpha)\theta^\alpha}{2\lambda e^{-(\DGam+ \vert t+\Delta\vert)/\theta}(\DGam+ \vert t+\Delta\vert)^{\alpha-1}}+\tfrac{\Gamma(\alpha)\theta^\alpha}{\Gamma(\alpha)\theta^\alpha\lambda}\frac{e^{- \vert t\vert/\theta}\frac{ \vert t\vert^\alpha}{\alpha}}{e^{-(\DGam+\vert t+\Delta\vert)/\theta}(\DGam+\vert t+\Delta\vert)^{\alpha-1}}\right)\\
    &\stackrel{(\ast\ast)}\le\tfrac{1}{\lambda c_Lc_{\DGam}}\left(\frac{e^{-\vert t\vert/\lambda}\Gamma(\alpha)\theta^\alpha}{ e^{-(\DGam+ \vert t\vert+\Delta)/\theta}(\DGam+ \vert t\vert+\Delta)^{\alpha-1}}+\frac{e^{- \vert t\vert/\theta}\frac{ \vert t\vert^\alpha}{\alpha}}{e^{-(\DGam+\vert t\vert+\Delta)/\theta}(\DGam+\vert t\vert+\Delta)^{\alpha-1}}\right)\\
    &=\frac{e^{(\DGam+\Delta)/\theta}}{\lambda c_Lc_{\DGam}}\left(\frac{e^{-\vert t\vert/\lambda}\Gamma(\alpha)\theta^\alpha}{ e^{-\vert t\vert/\theta}(\DGam+ \vert t\vert+\Delta)^{\alpha-1}}+\frac{e^{- \vert t\vert/\theta}\frac{ \vert t\vert^\alpha}{\alpha}}{e^{-\vert t\vert/\theta}(\DGam+\vert t\vert+\Delta)^{\alpha-1}}\right)\\
    &=\frac{e^{(\DGam+\Delta)/\theta}}{\lambda c_Lc_{\DGam}}\left(\Gamma(\alpha)\theta^\alpha\frac{(\DGam+ \vert t\vert+\Delta)^{1-\alpha}}{ e^{\vert t\vert(1/\lambda-1/\theta)}}+(\DGam+\vert t\vert+\Delta)^{1-\alpha}\frac{ \vert t\vert^\alpha}{\alpha}\right).
\end{align*}
where we at $(\ast)$ filled in the density functions and at $(\ast\ast)$ used that $\vert t+\Delta\vert\le \vert t\vert +\Delta$.

Using the identity $\vert t\vert^\alpha=\vert t\vert/\vert t\vert^{1-\alpha}$ we see that
\begin{align*}
    \frac{f_{A}(t)}{f_{A}(t+\Delta)}
    &\le \frac{e^{(\DGam+\Delta)/\theta}}{\lambda c_Lc_{\DGam}}\left(\Gamma(\alpha)\theta^\alpha\frac{(\DGam+ \vert t\vert+\Delta)^{1-\alpha}}{ e^{\vert t\vert(1/\lambda-1/\theta)}}+\frac{\vert t\vert}{\alpha}\left(\frac{\DGam+\vert t\vert+\Delta}{\vert t\vert}\right)^{1-\alpha}\right)\\
    &\le \frac{e^{(\DGam+\Delta)/\theta}}{\lambda c_Lc_{\DGam}}\left(\Gamma(\alpha)\theta^\alpha\frac{(\DGam+ \vert t\vert+\Delta)^{1-\alpha}}{ e^{\vert t\vert(1/\lambda-1/\theta)}}+\frac{\DGam+\vert t\vert+\Delta}{\alpha}\right)\\
    &\le \frac{e^{(\DGam+\Delta)/\theta}}{\lambda c_Lc_{\DGam}}\left(\Gamma(\alpha)\theta^\alpha\frac{(\DGam+ \vert t\vert+\Delta)^{1-\alpha}}{ e^{\vert t\vert(1/\lambda-1/\theta)}}+\frac{\DGam+\Du+\Delta}{\alpha}\right),
\end{align*}
where in the last two steps, we removed the $(1-\alpha)$-exponent on the last term and used that $ \vert t\vert\le \Du$.
Finally applying Lemma \ref{lem:maximizingfraction} (recall $0<\alpha<1$ and the assumption $1/\lambda-1/\theta\ge 1/(\Delta+\DGam)$) with $\kappa =0$ finishes the proof:
\begin{align*}
    \frac{f_{A}(t)}{f_{A}(t+\Delta)}
    &\le \frac{e^{(\DGam+\Delta)/\theta}}{\lambda c_Lc_{\DGam}}\left(\Gamma(\alpha)\theta^\alpha(\DGam+\Delta)^{1-\alpha}+\frac{\DGam+\Du+\Delta}{\alpha}\right).
\end{align*}
\end{proof}

\subsection{Proof of Lemma \ref{lem:DPratioGeneralt}}
\label{app:prooflemGeneralt}
\dpratioGeneralt*
\begin{proof}
We first give the intuition behind the proof:
Lemmas \ref{lem:DPratioLarget} and \ref{lem:DPratioSmallt} give upper bounds on the ratio for certain values of $t$, assuming $1/\lambda-1/\theta\ge \max\{1/(\Du+\DGam+\Delta),1/(\DGam+\Delta) \}=1/(\DGam+\Delta)$. An upper bound on both of these bounds simultaneously gives us a bound on the ratio, which holds for general $t>-\Delta/2$. 
We note that 
\begin{align*}
    1/\lambda-1/\theta\ge 1/(\DGam+\Delta)\quad \Leftrightarrow\quad \lambda\le \frac{\theta}{\frac{\theta}{\DGam+\Delta}+1},
\end{align*}
so if $\theta\le \DGam+\Delta$, then $\lambda\le \theta/2$ suffices.
Hence, our assumptions $\lambda\le \theta/2$, $\lambda\le \Delta/\ln(2)$, $\theta\le \DGam+\Delta$ and $t>-\Delta/2$ ensure that we can use Lemmas \ref{lem:DPratioLarget} and \ref{lem:DPratioSmallt}.
\medskip

So, by Lemmas \ref{lem:DPratioLarget} and \ref{lem:DPratioSmallt} we have for $-\Delta/2<t\in\mathbb{R}$
\begin{align*}
     \frac{f_{A}(t)}{f_{A}(t+\Delta)}&\le \frac{e^{(\DGam+\Delta)/\theta}(\Du+\DGam+\Delta)}{c_{\DGam}}\max\left\{\frac{1}{\alpha c_L\lambda}, \frac{1}{\Du}\right\}\\&\qquad+\frac{e^{(\DGam+\Delta)/\theta}}{c_{\DGam}}\frac{\Gamma(\alpha)\theta^\alpha}{\lambda c_L}\max\left\{(\DGam+\Delta)^{1-\alpha},\frac{c_{\Du}e^{\Du/\theta}}{2}(\DGam+\Delta+\Du)^{1-\alpha}\right\}.
\end{align*}
As, by assumption, $\theta\le \DGam+\Delta$, we see
\[
\theta^\alpha (\DGam+\Delta)^{1-\alpha}<\theta^\alpha (\Du+\DGam+\Delta)^{1-\alpha}<\Du+\DGam+\Delta
\]
 and so we may simplify to
\begin{align*}
     \frac{f_{A}(t)}{f_{A}(t+\Delta)}&\le \frac{e^{(\DGam+\Delta)/\theta}}{c_{\DGam}}\left(\max\left\{\tfrac{1}{\alpha c_L\lambda}, \tfrac{1}{\Du}\right\}(\Du+\DGam+\Delta)+\frac{\Gamma(\alpha)(\DGam+\Delta+\Du)}{\lambda c_L}\max\left\{1,\tfrac{c_{\Du}e^{\Du/\theta}}{2}\right\}\right)\\
     &= \frac{e^{(\DGam+\Delta)/\theta}(\Du+\DGam+\Delta)}{c_{\DGam}}\left(\max\left\{\frac{1}{\alpha c_L\lambda}, \frac{1}{\Du}\right\}+\frac{\Gamma(\alpha)}{\lambda c_L}\max\left\{1,\frac{c_{\Du}e^{\Du/\theta}}{2}\right\}\right).
\end{align*}

Let $\Du=\alpha\theta$ (i.e., the mean of the $\Gamma$-distribution). Recalling that by definition $c_L=1/4$ and by assumption $\lambda\le \theta/2$, so $\alpha c_L\lambda\le \alpha\theta/8<\alpha\theta=\Du$:
\begin{align*}
     \frac{f_{A}(t)}{f_{A}(t+\Delta)}&\le \frac{e^{(\DGam+\Delta)/\theta}(\alpha\theta+\DGam+\Delta)}{c_{\DGam}}\left(\frac{1}{\alpha c_L\lambda}+\frac{\Gamma(\alpha)}{c_L\lambda}\max\left\{1,\frac{c_{\Du}e^{\alpha}}{2}\right\}\right)\\
     &\le \frac{e^{(\DGam+\Delta)/\theta}(\alpha\theta+\DGam+\Delta)}{c_{\DGam}c_L\lambda}\left(\frac{1}{\alpha}+\Gamma(\alpha)e^{\alpha}\right),
\end{align*}
where the last step follows from the observation that $1\le c_{\Du}\le 2$ (recall $c_{\Du}$ was defined in Lemma \ref{lem:DPratioLarget}) and $e^\alpha>1$ for $\alpha>0$.

By assumption $\Gamma(\alpha)<1/\alpha$, then \[
1/\alpha+\Gamma(\alpha) e^\alpha\le \frac{2e^\alpha}{\alpha}
\]
and so we conclude
\begin{align*}
     \frac{f_{A}(t)}{f_{A}(t+\Delta)}&\le \frac{2e^{(\DGam+\Delta)/\theta}e^\alpha(\alpha\theta+\DGam+\Delta)}{\alpha c_{\DGam}c_L\lambda}.
\end{align*}
\end{proof}

\subsection{Proof of Lemma \ref{lem:parametersetting}}
\label{app:proofparamsetting}
\paramsetting*
\begin{proof}
Suppose $\vert t\vert\le \vert t+\Delta\vert$. The first inequality is satisfied as $f_A(t)\ge f_A(t+\Delta)$. Let $\DGam$ be as in Lemma \ref{lem:boundsonGammaGamma}. We turn to prove the latter inequality:
In order to apply Lemma \ref{lem:DPratioGeneralt} we make the following assumptions:
\begin{align}
\label{eq:assumptions}
\theta\le \DGam+\Delta,\quad \lambda\le \theta/2,\quad \lambda\le \Delta/\ln(2) \quad \text{and}\quad \Gamma(\alpha)\le 1/\alpha.
\end{align}
We choose parameters satisfying these assumptions towards the end of the proof.
\medskip

If $\DGam$ is at least the median of the $\Gamma$-distribution then $c_{\DGam}\ge 1/2$. So let $\DGam=\alpha\theta$ be the mean of the $\Gamma$-distribution (the mean is an upper bound on the median of the $\Gamma$-distribution \citep{Chen68}), to see 
\begin{align}
\label{eq:insert3delta}
\frac{f_{A}(t)}{f_{A}(t+\Delta)}\stackrel{\text{(Lemma \ref{lem:DPratioGeneralt})}}\le \frac{2e^{(\Delta+\DGam)/\theta}e^{\alpha}(\alpha\theta+\DGam+\Delta)}{\alpha c_{\DGam}c_L\lambda}\le \frac{2\cdot (2\alpha\theta+\Delta)}{1/2\cdot 1/4}\frac{e^{\Delta/\theta}e^{2\alpha}}{\alpha \lambda}.
\end{align}
Suppose $\alpha\le 1/2$ and recall by assumption $\theta\le \DGam+\Delta=\alpha\theta+\Delta$, so $\theta\le \Delta/(1-\alpha)$. Then $\theta\le \Delta/(1-\alpha)\le \Delta/\alpha$ and so $\alpha\theta\le \Delta$. 
We revise our set of assumptions, to also ensure that $\alpha\theta\le \Delta$, and so our set of assumptions is:
\begin{align}
\label{eq:assumptions2}
\theta\le \frac{\Delta}{1-\alpha},\quad \lambda\le\min\{ \theta/2,\Delta/\ln(2)\},\quad \alpha\le 1/2 \quad \text{and}\quad \Gamma(\alpha)\le 1/\alpha.
\end{align}
Under these assumptions we have $2\alpha\theta+\Delta\le 3\Delta$ and inserting into (\ref{eq:insert3delta}), we conclude
\[
\frac{f_{A}(t)}{f_{A}(t+\Delta)}\le 48\Delta\frac{e^{\Delta/\theta}e^{2\alpha}}{\alpha \lambda}.
\]
Now define
\[
\alpha=1/e^{\varepsilon/k_\alpha},\quad \theta=\frac{k_\theta\Delta}{\varepsilon}\quad \text{and}\quad \lambda=1/e^{\varepsilon/k_\lambda}
\]
Observing $2\alpha\le 1$ and $\ln(48e)<4.9$
\[
\frac{f_{A}(t)}{f_{A}(t+\Delta)}\le 48e^{2\alpha}\Delta e^{\varepsilon(1/k_\theta+1/k_\alpha+1/k_\lambda)}< e^{\varepsilon(1/k_\theta+1/k_\alpha+1/k_\lambda)+4.9+\ln(\Delta)}.
\]
Hence, we ensure that 
\begin{align*}
\frac{f_{A}(t)}{f_{A}(t+\Delta)}\le e^\varepsilon
\end{align*}
when the assumptions in (\ref{eq:assumptions2}) are satisfied and
\begin{align}
\label{eq:restrforks}
\varepsilon(1/k_\theta+1/k_\alpha+1/k_\lambda)+5+\ln(\Delta)\le \varepsilon\quad \Leftrightarrow\quad 1/k_\theta+1/k_\alpha+1/k_\lambda\le 1-\frac{5+\ln(\Delta)}{\varepsilon}.
\end{align}

Let $k_\alpha=k_\lambda=k_\theta=4$. It is easy to check that the assumptions on $\theta$ and $\alpha$ in (\ref{eq:assumptions2}) are satisfied simultaneously for, $\varepsilon\ge 4\ln(2)$ (and we can check that $\Gamma(\alpha)\le 1/\alpha$ numerically). Furthermore, for $\varepsilon\ge 4\ln(2)$, we require $\lambda\le \Delta \min\{2/\varepsilon, 1/\ln(2)\}=2\Delta/\varepsilon$ and so the assumption on $\lambda$ is satisfied when $\Delta\ge \varepsilon\lambda/2=\varepsilon e^{-\varepsilon/4}/2 $. Observing that $\Delta\ge 2/e\ge \varepsilon e^{-\varepsilon/4}/2$, we conclude that the assumptions in (\ref{eq:assumptions2}) are satisfied for $\varepsilon\ge 4\ln(2)$ and $\Delta\ge 2/e$. The inequality in (\ref{eq:restrforks}) is satisfied for 
\[
3/4\le 1-\frac{5+\ln(\Delta)}{\varepsilon}\quad \Leftrightarrow\quad 20+4\ln(\Delta) \le \varepsilon.
\]
Finally, observe that if $\vert t\vert\ge \vert t+\Delta\vert$, the result follows by symmetry of $f_A$ (Corollary \ref{cor:symmetryArete}). 
\end{proof}

\end{document}